\documentclass[conference]{IEEEtran}
\usepackage{subfigure}
\usepackage{graphicx,amsmath,amsthm} % Add all your packages here
\usepackage{times}
\usepackage{bbm,bm}
\usepackage{algorithm}
\usepackage{multirow}
\usepackage{booktabs}
\usepackage{algorithmic}
\usepackage{makecell}
\usepackage{bigstrut}
\usepackage{bbold}
\usepackage{color}

\usepackage{helvet}
\usepackage{courier}
\usepackage{subfigure}
\usepackage{comment}
\usepackage{array}
\usepackage{amssymb}
\usepackage{slashbox}

\newtheorem{definition}{Definition}
\newtheorem{theorem}{Theorem}

\newcommand{\ugg}{\fontsize{9.5pt}{9.5pt}\selectfont}

\newcommand{\tgg}{\fontsize{7.5pt}{7.5pt}\selectfont}
\begin{document}
\title{PIGMIL: Positive Instance Detection via Graph Updating for Multiple Instance Learning}

%\author{}
%\author{\IEEEauthorblockN{Dongkuan Xu,
%Jia Wu,
%Wei Zhang
%Yingjie Tian
%\IEEEauthorblockA{Centre for Quantum Computation and Intelligent Systems, FEIT, \\
%University of Technology Sydney, NSW 2007, Australia}
%\fontsize{9}{9}\selectfont\ttfamily\upshape
%Qin.Zhang@student.uts.edu.au; Jia.Wu@uts.edu.au; zw-info@ruc.edu.cn; tyj@ucas.ac.cn }
%}

%\address[a]{School of Mathematical Sciences, University of Chinese Academy of Sciences, Beijing 100049, China}
%\address[b]{Research Center on Fictitious Economy and Data Science, Chinese Academy of Sciences, Beijing 100190, China}
%\address[c]{School of Information, Renmin University of China, Beijing 100872, China}
%\address[d]{Key Laboratory of Big Data Mining and Knowledge management, Chinese Academy of Sciences, Beijing 100190, China}

\author{\IEEEauthorblockN{Dongkuan Xu\IEEEauthorrefmark{1},
Jia Wu\IEEEauthorrefmark{2},
Wei Zhang\IEEEauthorrefmark{3}, and
Yingjie Tian\IEEEauthorrefmark{4}\IEEEauthorrefmark{5}}
\IEEEauthorblockA{\IEEEauthorrefmark{1} School of Mathematical Sciences, University of Chinese Academy of Sciences, Beijing 100049, China\\
Email: xudongkuan14@mails.ucas.ac.cn}
\IEEEauthorblockA{\IEEEauthorrefmark{2} QCIS, University of Technology, Sydney, NSW 2007, Australia\\
Email: jia.wu@uts.edu.au}
\IEEEauthorblockA{\IEEEauthorrefmark{3} School of Information, Renmin University of China, Beijing 100872, China\\
Email: zw-info@ruc.edu.cn}
\IEEEauthorblockA{\IEEEauthorrefmark{4} Research Center on Fictitious Economy and Data Science, Chinese Academy of Sciences, Beijing 100190, China\\}
\IEEEauthorblockA{\IEEEauthorrefmark{5} Key Laboratory of Big Data Mining and Knowledge management, Chinese Academy of Sciences, Beijing 100190, China\\
Email: tyj@ucas.ac.cn}}

\maketitle

\begin{abstract}

Positive instance detection, especially for these in positive bags (true positive instances, {\it TPI}s), plays a key role for multiple instance learning ({\it MIL}) arising from a specific classification problem only provided with bag (a set of instances) label information. However, most previous {\it MIL} methods on this issue ignore the global similarity among positive instances and that negative instances are non-i.i.d., usually resulting in the detection of {\it TPI} not precise and sensitive to outliers.
To the end, we propose a positive instance detection via graph updating for multiple instance learning, called {\it PIGMIL}, to detect {\it TPI} accurately. {\it PIGMIL} selects instances from working \ sets ($\mathcal{WS}s$) of some working bags ($\mathcal{WB}s$) as positive candidate pool ({\it PCP}). The global similarity among positive instances and the robust discrimination of instances of {\it PCP} from negative instances are measured to construct the consistent similarity and discrimination graph ({\it CSDG}).
As a result, the primary goal (i.e. {\it TPI} detection) is transformed into {\it PCP} updating, which is approximated efficiently by updating {\it CSDG} with a random walk ranking algorithm and an instance updating strategy. At last bags are transformed into feature representation vector based on the identified {\it TPI}s to train a classifier. Extensive experiments demonstrate the high precision of {{\it PIGMIL}}'s detection of {\it TPI}s and its excellent performance compared to classic baseline {\it MIL} methods.
\end{abstract}


\IEEEpeerreviewmaketitle

\section{Introduction}

Multiple instance learning (MIL) was formally proposed for drug activity detection \cite{dietterich1997solving} at first. Contrary to traditional classification problem, MIL deals with bag, or set of instances, classification to label a bag positive or negative where not all instance label information is exploit. Based on the general MIL setting, a bag is labelled positive if it contains at least one positive instance, or else it is considered as a negative one. However, the specific label of individual instance in positive bags is unknown. Because of its ability to cope with instance label ambiguity, MIL has been applied into various applications in pattern recognition and computer vision, e.g.,
image categorization \cite{chen2006miles,shen2016multiple},
object detection \cite{zhang2005multiple,yi2015human},
graph classification \cite{wu2013multi,wu2014bag,wu2015boosting},
text categorization \cite{andrews2002support,zhou2009multi,rastegari2015discriminative},
%web mining \cite{zhou2005multi,li2011web},
%information retrieval \cite{rousu2005learning},
etc.

\begin{figure}[!t]\vspace{0.1cm}
\centering
\includegraphics[width=1\linewidth, height=6cm]{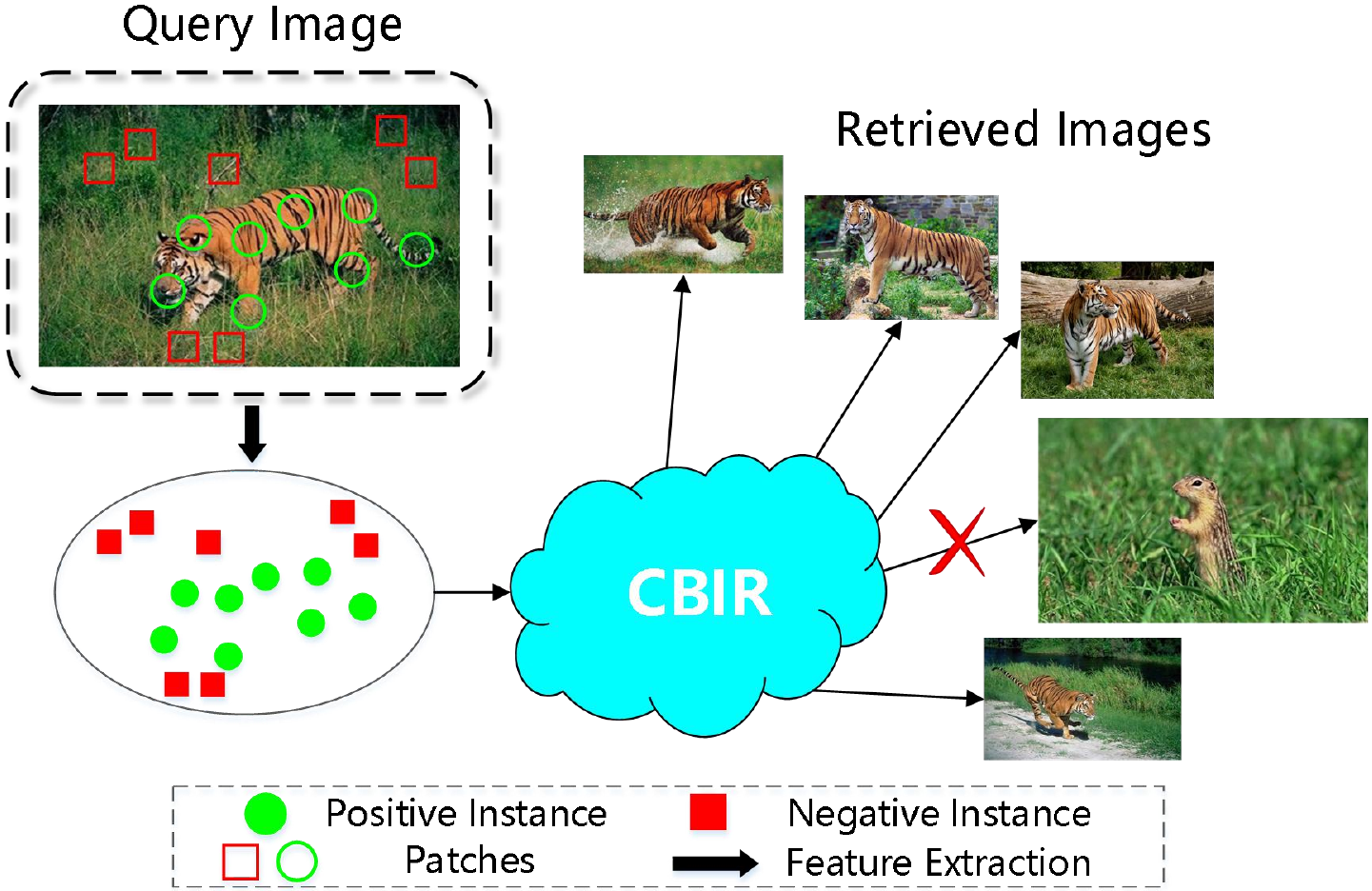} \vspace{-0.4cm}
\caption{Positive instance plays an important role for CBIR based on MIL. An image of \emph{tiger} is divided into some patches, each patch corresponds to an instance, and this image is considered as a positive bag for \emph{tiger}. Patches involved in \emph{tiger} correspond to positive instances as well as TPIs for this image. CBIR takes use of these instances to search retrieved images for a query image. And there may also be an irrelevant one, like the one of \emph{squirrel}, in the retrieved images because of FPIs.}\vspace{-0.5cm}
\label{fig:tiger}
\end{figure}

The positive instances in positive bags are called the true positive instance denoted as TPIs, with the negative instances in positive bags (false positive instances) denoting FPIs.
The intrinsic problem of MIL is to determine whether a bag contains TPIs or not. The typical application of {TPI}'s detection is content-based image retrieval (CBIR) \cite{li2009convex}, of which the main objective is to locate the regions of interest (ROIs) in images that show a great discriminative ability to label images. As shown in Figure \ref{fig:tiger}, the image with $tiger$ is divided into several patches based on feature extraction methods. According to the MIL framework, the whole image is considered as a positive bag and each patch is taken as an instance. The patches involved in $tiger$, called TPIs, corresponds to ROIs and are significant in the image retrieval for $tiger$. The rest patches are called FPI providing little information for CBIR.
There are extensive studies on TPIs of MIL \cite{dietterich1997solving,wang2000solving,zhang2001dd,andrews2002support,chen2006miles,li2010mild,fu2011milis,rastegari2015discriminative}.
APR \cite{dietterich1997solving} constructs a axis-parallel rectangle to encompass instances from different positive bags as many as possible while minimizing the number of instances from negative bags. The rectangle is considered as the region where TPIs are located.
DD-based method \cite{zhang2001dd} extends the basic idea of APR, tries to recognize the instances with high $diverse \ density$ value, i.e., instances near all positive bags while distant from negative bags, and regards these instance as TPIs.
SVM-based methods \cite{andrews2002support,li2009convex} utilize SVM to discriminate TPIs.
mi-SVM \cite{andrews2002support} searches for a hyperplane at instance-level where each positive bag has at least one instance located in positive space while all negative {bags}' instances are in the negative. KI-SVMs \cite{li2009convex} proposes two different level convex optimization models based on SVM and maximizes the margin by the most violated key instance to locate key instances at different levels.
MILD \cite{li2010mild} focuses on the ambiguous information of instances in the positive bags. It selects an instance with the highest maximum empirical precision of each positive bag as the TPI and constructs a two-level classification scheme based on the selected TPIs inspired by MILES \cite{chen2006miles}.
MILIS \cite{fu2011milis} first selects instance prototypes (IPs) by Gaussian-kernel-based kernel density estimator on negative instances, then updates these IPs, trains classifier in an iterative learning framework to construct the feature representation for each bag, and employs the SVM to classify a new bag at last.
%mi-Sim \cite{rastegari2015discriminative}

\begin{figure}[!t]\vspace{0.5cm}
\centering
\includegraphics[width=0.95\linewidth, height=5cm]{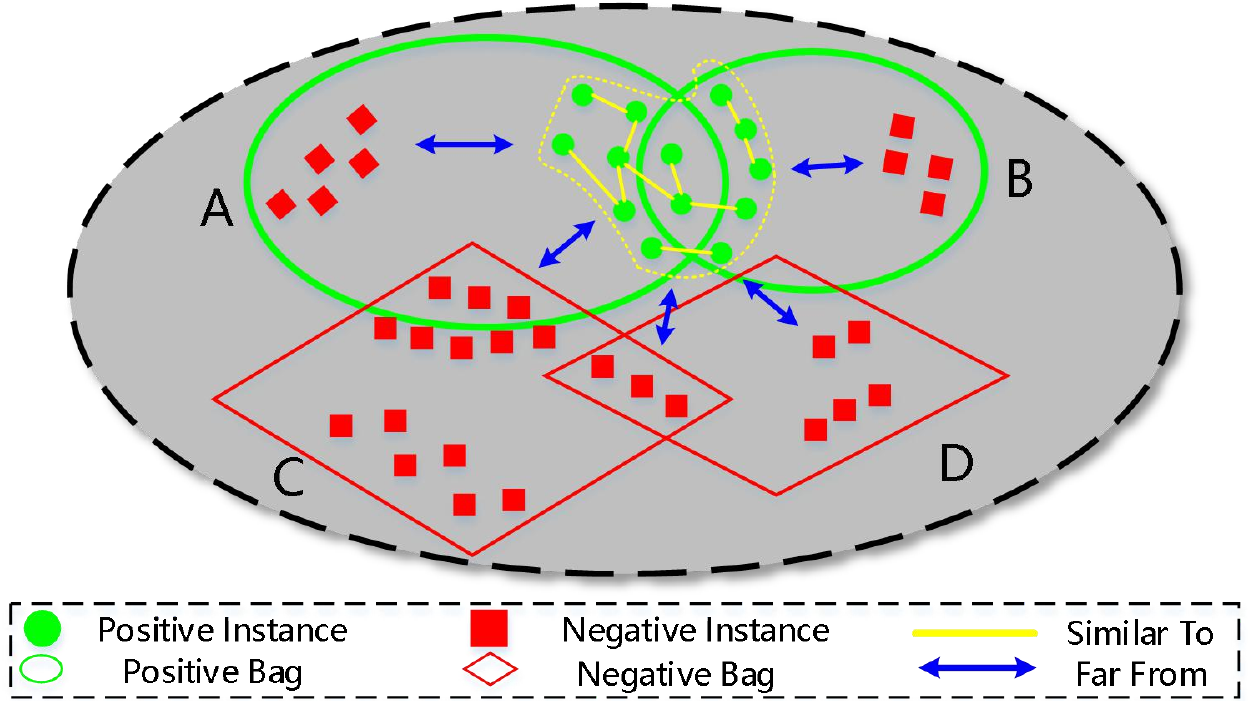} \vspace{-0.1cm}
\caption{Four bags are represented in feature space. Bag A, B are positive represented as green ellipses, and bag C, D are negative represented as red rhombuses. The true positive instances (TPIs) represented as green solid circles and compassed by the polygonal yellow line in bag A and B should be far from negative instances, which are represented as the blue double-headed arrows. Moreover, TPIs should also be similar to each other, which is represe nted as the straight yellow lines.}\vspace{-0.5cm}
\label{fig:sim_dis}
\end{figure}

However, these common MIL methods on identifying TPIs have some disadvantages.
APR only showed high performance for drug activity detection because it is hard to construct such a rectangle accurately for data sets in other application context.
DD-based methods \cite{zhang2001dd} are sensitive to noise, which means the $diverse \ density$ value will decrease dramatically if there are some negative instances nearby. Moreover, DD-based methods need to consider each instance in positive bags resulting a high computation cost.
MILD \cite{li2010mild} simply considers the instance with the highest empirical precision in each positive bag as the TPI. The empirical precision is calculated based on all training bags and a threshold $\theta_t$ which is hard to determined.
Generally, most TPI detection methods for MIL do not consider the similarity among TPIs or utilize it in depth.
Similarity among TPIs possesses the great significance on the TPI detection because it reveals the intrinsic property of TPIs while it may result from some coincidental patterns that are not irrelevant to the topic \cite{shrivastava2011data}. For instance, when we want to judge whether two images are similar because of the target content or not, these two images may be similar for sharing the irrelevant contents. These contents correspond to coincidental patterns which are not repetitive in feature space. This implies that a reliable similarity should be homogeneous across several parts, i.e., a global similarity.
Moreover, the discrimination between TPIs and negative instances is not robust to outliers for most methods.
Discrimination between TPIs and negative instances provides us an reliable way to decide whether an instance is negative or not because only the negative {instance}'s label is determined in MIL.
Although most TPI detection methods utilized the difference between positive and negative instances, but the influence of a far negative instance and a near negative on a TPI are not sufficiently characterized respectively. The influence of a far negative instance on an instance {$x$}'s label should decrease exponentially when it becomes farther from $x$, and a near negative instance should increase its influence on {$x$}'s label sharply when it becomes closer to $x$, i.e., a robust discrimination.
Furthermore, it is unnecessary to search all instances of all positive bags to identify TPIs while there are at least one TPI in each positive bag. This is because computation cost is too high to search all positive bags, not every instance in positive bags is positive, and TPIs from some positive bags may not be positive enough.

Inspired by these observations, this paper proposes a \underline{P}ositive \underline{I}nstance detection via \underline{G}raph updating for \underline{M}ultiple \underline{I}nstance \underline{L}earning (PIGMIL) whose core idea is to identify TPIs that should not only be similar to themselves globally but also different from negative instances robustly shown in Figure \ref{fig:sim_dis}.
PIGMIL determines and initializes $working \ sets$ ($\mathcal{WS}s$), $working \ bags$ ($\mathcal{WB}s$), and positive candidate pool (PCP) at first to reduce the computation cost and improve the accuracy of TPI detection.
The original TPI detection is approximated by maximizing global similarity among positive instances and robust discrimination of positive instances from negative ones based on PCP.
Then the maximum optimization problem is dealt with on a consistent similarity and discrimination graph ($CSDG$) with a random walk algorithm and an instance updating strategy.
Bags are embedded into instance-based feature space and transformed into representation vectors by TPIs to train the classifier.

The main contributions of PIGMIL are summarized:
\begin{itemize}
\item[1)] The global similarity among positive instances is utilized. Combining the $similarity$ ($\mathcal{S}$) and its $consistency$ ($\mathcal{C}$) provides a global similarity ($\mathcal{S}$+$\mathcal{C}$) and avoids the misleading of coincidental patterns on TPI detection.
\item[2)] The robust discrimination of positive instances from negative instances is exploited. The $discrimination$ ($\mathcal{D}$) is robust to outliers, decreasing a far {negative}'s influence sharply when it gets farther away from an instance and putting exponentially more importance on near negative instances if they become closer to an instance.
\item[3)] $\mathcal{WS}$, $\mathcal{WS}$, and PCP are determined to reduce computation cost and improve searching accuracy. The original objective of identifying TPIs is transformed into PCP updating and then approximated efficiently by updating graph $CSDG$ iteratively with an instance updating strategy..
\end{itemize}

In the rest of paper, we: define basic concepts and give an overview of PIGMIL in Section~\ref{pf}; describe PIGMIL at length in Section~\ref{pm}; conduct experiments in Section~\ref{ex}; make discussion in Section~\ref{dis}; and draw conclusion in Section~\ref{conl}.

\section{Problem Formulation}\label{pf}

In this section, we define some important notations, then provide a formal definition of the MIL problem.

\begin{definition}{\emph{(Instance and Bag)}}
Let $x_i$ = $(x_1, \cdots, x_d)^T$ and $X_j$ = $(x_{j1}, \cdots, x_{jn_j})$ denote an instance and a bag separately, where $d$ is the dimensionality, $n_j$ is the number of instances of $X_j$, and $x_{jk}$ is an instance belonging to $X_j$. Each instance and bag are labelled with $y \in \{+1,-1\}$ and $L \in \{+1,-1\}$ separately. $+1$ indicates the instance or bag is positive and $-1$ corresponds to negative \cite{andrews2002support}.
\end{definition}

\begin{definition}{\emph{(KDE$_{min}$)}}
Based on KDE \cite{duda2012pattern}, KDE$_{min}$ is defined as:
\begin{equation}\label{KDEmin}
\small
f_{KDE_{min}}(x) = \frac{1}{Z \times N^-} \sum\limits_{L_j=-1} \min\limits_{x_{ji} \in X_j} exp(-\gamma \|x - x_{ji}\|)
\end{equation}
where $N^{-}$ is the number of negative bags, and $Z,$ $\gamma$ are empirical parameters.
\end{definition}

%\begin{definition}{\emph{(Gaussian-kernel-based Kernel Density Estimator (KDE))}}
%The density value of instance $x$ based on KDE is:
%\begin{equation}\label{KDE}
%\small
%f_{KDE}(x) = \frac{1}{Z\sum_{j=1}^{N^-}n_j} \sum\limits_{L_j=-1} \sum\limits_{i=1}^{n_j}exp(-\gamma \|x - x_{ji}\|)
%\end{equation}
%where $N^{-}$ is the number of negative bags, $n_j$ is the number of instances in bag $X_j$, $L_j=-1$ indicates bag $X_j$ is negative, and $Z,\gamma$ are empirical parameters \cite{duda2012pattern}.
%\end{definition}

\begin{definition}{\emph{(Working Set)}}
 The working set of bag $X_j$ is represented as $ \mathcal{WS}_j \in \{(x_1, \cdots, x_{n_j}) \ | \ ws(x_k) \leqslant T_{ws_j}, \forall k \in (1, \cdots, n_j) \}$, where $ws(\cdot)$ represents a decision function to decide whether an instance belongs to $\mathcal{WS}$ or not, $T_{ws_j}$ represents a threshold, and $n_j$ represents the size of $\mathcal{WS}_j$.
\end{definition}

\begin{definition}{\emph{(Working Bag)}}
 A positive bag $X_j$ is called a working bag represented as $ \mathcal{WB}_j $, iff the values of instances in $\mathcal{WS}_j$ based on the decision function $ws(\cdot)$ is not significantly worse than the values of instances in other positive {bags}' working sets.
\end{definition}

\begin{definition}{\emph{(Positive Candidate Pool)}}
  A positive candidate pool is a group of instances represented as $ PCP = \{x_{wb_1}^*, \cdots, x_{wb_{n_w}}^*\}$, where $x_{wb_j}^*$ is an instance from the $\mathcal{WS}$ of $\mathcal{WB}_{wb_j}$, and $n_w$ is the number of working bags.
\end{definition}

\begin{figure*}[!t]
\centering
\includegraphics[scale =0.75]{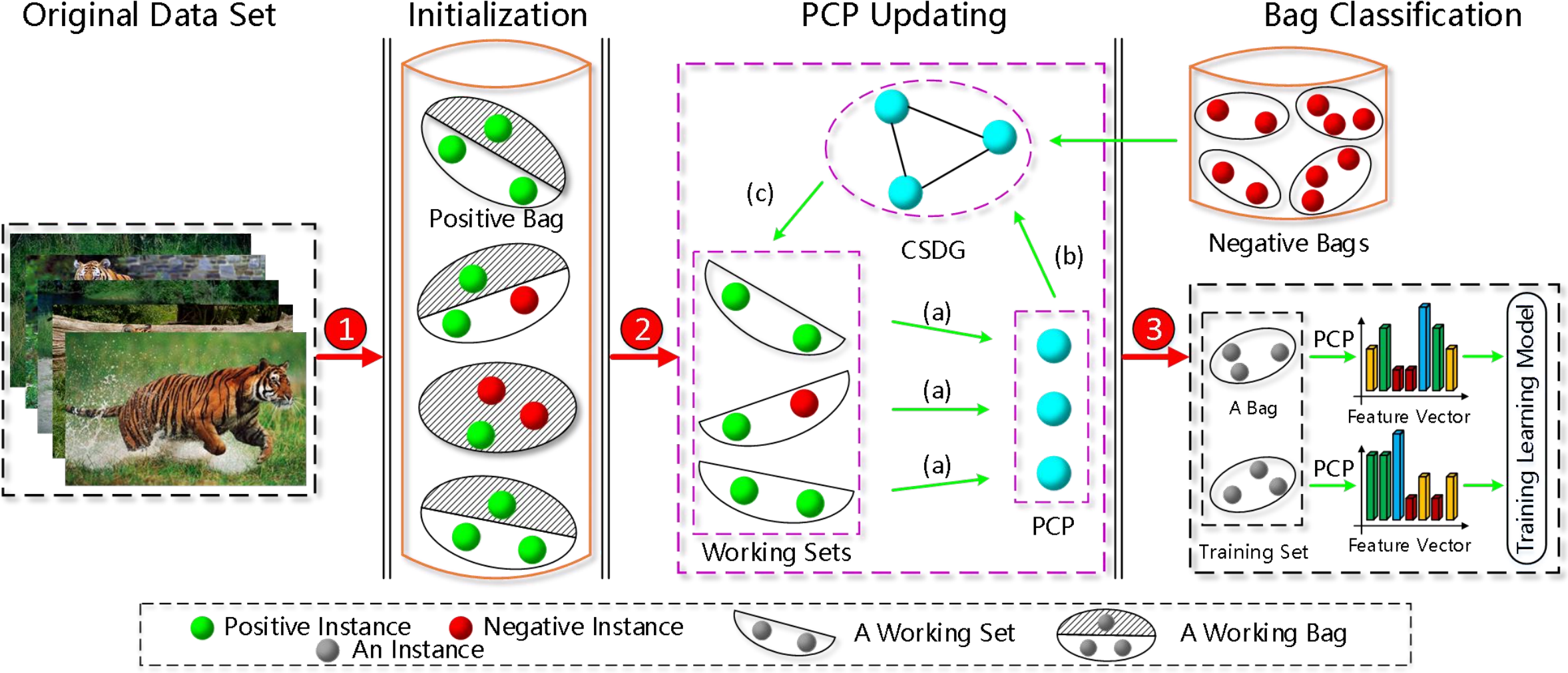}\vspace{-0.2cm}
\caption{A conceptual view of Positive Instance Detection via Graph Updating for Multiple Instance Learning (PIGMIL): The goal of PIGMIL is to construct a bag classification scheme to label a new bag $\footnotesize{\textcircled{\scriptsize{3}}}$ based on the instances in updated positive candidate pool (PCP) $\footnotesize{\textcircled{\scriptsize{2}}}$ after $working \ sets$ $(\mathcal{WS}s)$, $working \ bags$ $(\mathcal{WB}s)$, and PCP are initialized from original data set $\footnotesize{\textcircled{\scriptsize{1}}}$.
Specifically, original data set is preprocessed into bags (sets of instances) based on MIL at first. To improve the accuracy and reduce computation cost of searching the true positive instances (TPIs), $\mathcal{WS}s$, $\mathcal{WB}s$, and PCP (consisting of one instance from each $\mathcal{WS}$ (a)) are identified and initialized.
Instances in PCP are considered to be positive. Then to discern the instance $x_t$ in PCP that is not positive and needed to replace, the consistent similarity and discrimination graph (CSDG) is built (b). $x_t$ is identified by a random walk algorithm on CSDG and updated by an instance updating strategy (c). Eventually a bag is classified by a bag classification scheme, where bags are embedded into a updated PCP-based feature space and transformed into feature vectors to train a SVM classifier.}
\label{fig:overview}\vspace{-0.3cm}
\end{figure*}

Given a group of bags as $X = \{X_1, \cdots, X_N\}$, where each positive bag consists of at least one positive instance while all instances are negative in negative bags. The objective of MIL is to build a classification model based on training set only with bag labels to predict the labels of new bags.
The overall framework includes three major steps presented in Figure \ref{fig:overview}:

\begin{itemize}
\item\textbf{Initialization:} \textbf{1.} To improve the efficiency of updating PCP and reduce computation cost, $\mathcal{WS}s$ and $\mathcal{WS}s$ are initialized at first. By doing so, we can update PCP from the most possibly positive instances. \textbf{2.} We take one instance from the $\mathcal{WS}$ of each $\mathcal{WB}$ based on KDE$_{min}$ to initialize PCP.

\item\textbf{PCP Updating:} To maximize the global similarity among instances in PCP and the robust discrimination of these instances from negative ones, $CSDG$ is constructed to recognize the instance in PCP that shares the least similarity with other instances and least difference from negative instances with a random walk algorithm, and replace it by a new one according to an instance updating strategy.
\item\textbf{Bag Classification:} To label unknown bags, a bag classification scheme based on updated PCP is constructed. The distance between a bag and each instance of updated PCP is exploited to transform the bag into feature presentation vectors. A SVM classifier is learned on the vectors.

\end{itemize}

\section{The proposed method PIGMIL}\label{pm}

\subsection{Initialization}\label{pm:sub1}

Initialization refers to initialize $\mathcal{WS}s$, $\mathcal{WB}s$, and PCP.
It is necessary to determine the useful positive bags and their useful instance candidates to construct PCP for identifying TPIs because the computation cost is too high to search from all positive bags and not every instance in positive bags is positive enough or positive actually.

\subsubsection{\textbf{Working Set}}\label{pmsub1:sub1}

A $working \ set$ ($\mathcal{WS}$) refers to the useful instance candidates for a positive bag.
Instances in $\mathcal{WS}$ are with high possibility to be TPIs. We take advantage of negative instances to figure our this possibility because only the labels of instances in negative bags are known. However, negative instances may share very general distributions, we adopt KDE$_{min}$ (Eq. (\ref{KDEmin})) as decision function $ws(\cdot)$ to capture the relationship between a instance and negative ones to construct $\mathcal{WS}$. In other words, instance $x_i$ of bag $X_j$ belongs to $\mathcal{WS}_j$ if $f_{KDE_{min}}(x_i) \leqslant T_{ws_j}$.

\subsubsection{\textbf{Working Bag}}\label{pmsub1:sub2}

$Working \ bags$ ($\mathcal{WB}s$) correspond to the seleceted positive bags that are used to update PCP.
To determine all $\mathcal{WB}s$ from all positive bags, $T$-test \cite{winer1971statistical} is employed to check whether the average value of instances in a positive {bag}'s $\mathcal{WS}$ is significantly worse than the average value of all instances in rest positive {bags}' $\mathcal{WS}$ based on $ws(\cdot)$. If it is, this positive bag will not be considered as a $\mathcal{WB}$. In this paper, $f_{KDE_{min}}(\cdot)$ is chosen as $ws(\cdot)$.

\subsubsection{\textbf{Positive Candidate Pool}}\label{pmsub1:sub3}

$Positive\ candidate \ pool$ (PCP) includes some instances from the $\mathcal{WS}s$ of $\mathcal{WB}s$ and only one instance is chosen from a $\mathcal{WS}$. Instances in PCP are considered as positive ones, which means they should share high similarity among themselves and significant difference from negative instances.
PCP is used to construct a bag classification scheme after it is updated, i.e., instances in PCP are positive enough. Initially, the instance $x$ with the lowest $ws(\cdot)$ in $\mathcal{WS}_{wb_j}$ of $\mathcal{WB}_{wb_j}$ is chosen as $x_{wb_j}^*$.

\subsection{PCP Updating}\label{pm:sub2}

Some instances in the initialized PCP are not positive enough or not positive actually.
PCP updating refers to that instances in PCP are updated to be positive enough in general, i.e., sharing high similarity among themselves and great difference from negative ones. But the original updating PCP is a difficult combinational optimization problem. So we transform it into an approximation based on consistent similarity and discrimination graph (CSDG). Additionally, an instance updating strategy is proposed to accelerate the updating.

\subsubsection{\textbf{Optimization Objective of Updating PCP}}\label{pmsub2:sub1}

The goal of updating PCP is to maximize the overall similarity of instances in PCP and their difference from negative instances. However, it is a challenging task for most learning problems to learn the overall similarity directly. We adopt a kind of pairwise similarity $\mathcal{S}$ to approximate it. To improve the approximation, the consistency $\mathcal{C}$ for $\mathcal{S}$ is employed to discriminate $\mathcal{S}$ that is homogeneous across different parts. The difference of an instance from negative instances is represented as $\mathcal{D}$.
Therefore, the original goal of updating PCP can be formulated to find the best labeling $\mathcal{L}$ for training instances to maximize $\mathcal{S}$, $\mathcal{C}$, and $\mathcal{D}$ for the instances in PCP:
\begin{equation}
\label{goal-1}
\begin{split}
&\max_{\mathcal{L}} \sum_{(x_i,x_j),x_k \in PCP}  \alpha\mathcal{S}(x_i,x_j) + \mathcal{C}(x_i,x_j) + \beta\mathcal{D}(x_k)\\
\end{split}
\end{equation}
where $(x_i,x_j)$ is a pair of instances in PCP, $x_k$ is an instance in PCP, $\alpha$, $\beta$ are balancing factors, and $\mathcal{S}$, $\mathcal{C}$, $\mathcal{D}$ are $similarity$, $consistency$, $discrimination$ respectively. $\mathcal{S}$ + $\mathcal{C}$ indicates the global similarity and $\mathcal{D}$ indicates the robust discrimination.

\textbf{Similarity:}
Because only the labels of negative instances are known, we calculate the similarity between two instances $\mathcal{S}(x_i, x_j)$ based on how similarly different they are from negative instances. Inspired by \cite{rastegari2015discriminative}, we use $x_i$ as a positive instance and all the negative instances to learn a classifier based on SVM. The unbalance of positive instances and negative ones is coped with by resampling $x_i$. The confidence of $x_j$ based on the learned classifier is $\Upsilon _{i,j}  = w_i^T \cdot x_j$, where $w_i$ is the learned weight based on $x_i$.

\begin{definition}{\emph{(Similarity)}}
    The similarity between instance $x_i$ and $x_j$ is:
    \begin{equation}\label{simil}
    \mathcal{S}(x_i, x_j)=\left\{
    \begin{array}{rcl}
    \frac{1}{\varphi(i,j) \cdot \varphi(j,i)}  & {if\ \Upsilon_{i,j} > 0 \ and \ \Upsilon_{j,i} > 0}\\
    0            & {otherwise}
    \end{array} \right.
    \end{equation}
    where $\varphi(j,i)$ stands for the order of $x_j$ among other instances whose confidence is positive when they are classified by $w_i$.
\end{definition}

\textbf{Consistency:}
To improve the accuracy of similarity, the consistency for each pairwise similarity is figured out.
Sometimes the similarity between two objects may be confused for coincidental patterns. Therefore, the intrinsical similarity should be consistent across several parts.
In this paper, the size of the maximal quasi-clique including the two instances is adopted as $consistency$ for their similarity:

\begin{definition}{\emph{(Consistency)}}
  In a graph $Graph = (V,E)$, the consistency for $v_i$ and $v_j$ is the size of the maximal quasi-clique and defined as:
\begin{equation}\label{consis}
\mathcal{C}(v_i, v_j)=\left\{
\begin{array}{rcl}
\max \limits_{k}\{|Q_k|\}  & {\forall k \ v_i, v_j \in Q_k}\\
0                          & {\not\exists k \ v_i, v_j \in Q_k}
\end{array} \right.
\end{equation}
where $Q_k$ represents different maximal quasi-cliques consisting of $v_i$ and $v_j$.
\end{definition}

A quasi clique corresponds to a undirected graph $Graph = (V,E)$, where $|E|\geqslant \left \lfloor \gamma \binom{|V|}{2} \right \rfloor$ and $0 < \gamma \leqslant 1$ \cite{brunato2007effectively}. In this paper, we set $\gamma$ to be 0.9.
The vertexes in quasi-clique share dense similarities among themselves.
And the maximal quasi-clique is a quasi-clique when there is no node can be added to extend the quasi-clique.
In other words, the size of the maximal quasi-clique for two objects is large when the similarities of two objects are consistent, i.e., existing several homogeneous similarities.

\textbf{Discrimination:}
Beyond that positive instances should be similar to themselves, positive instances should also be different from negative ones.
To measure the difference between an instance and negative instances, inspired by Gaussian-kernel-based kernel density estimator (KDE) \cite{duda2012pattern}, the discrimination of an instance is defined as:

\begin{definition}{\emph{(Discrimination)}}\label{Discrimination}
%\small
  The discrimination of instance $x_i$ from other negative instances is:
\begin{equation}\label{discri-1}
\small
\mathcal{D}(x) = \frac{1}{Z\sum\limits_{j=1}^{N^-}n_j} \sum\limits_{L_j=-1} \sum\limits_{i=1}^{n_j}d(x,x_{ji})
\end{equation}

\begin{equation}\label{discri-2}
\small
d(\Delta)=\left\{
\begin{array}{rcl}
-exp[-\gamma_1 (\Delta-1)]  & {\Delta \geqslant 1}\\
\gamma_2 ln\Delta - 1  & {1 > \Delta > 0}\\
-\infty & {\Delta = 0}
\end{array} \right.
\end{equation}
where $\Delta = \Delta(x,x_{ji})$ is a distance function between $x$ and $x_{ji}$, $N^{-}$ is the number of negative bags, $n_j$ is the number of instances in bag $X_j$, $L_j=-1$ indicates bag $X_j$ is negative, and $Z,\gamma_1, \gamma_2$ are positive empirical parameters.
\end{definition}

In general, when how likely $x_i$ is negative based on the known negative instances is to be determined, we should consider: 1) the closer a negative instance is for $x_i$, the more influence it has on {$x_i$}'s label. 2) the far negative instances should not show much influence. In other words, $\mathcal{D}$ should be robust to outliers.
%Figure \ref{} represents a toy example demonstrating the two aspects.

\begin{theorem} %\textbf{XXXXXX:}
 The influence of far / close negative instances on $\mathcal{D}$ defined in Eq. (\ref{discri-1}) is decreased / increased sharply when the negative instances are farther / closer and the influence of a far negative instance is limited.
\end{theorem}

\begin{proof}
For an instance $x_i$ and a negative one $x_j$, it is supposed that $x_j$ is a far negative instance for $x_i$ if $\Delta = \Delta(x_i,x_j) \geqslant 1$; otherwise, $x_j$ is a close one for $x_i$.

If $x_j$ is a far negative instance for $x_i$, i.e., $\Delta \geqslant 1$.
According to the definition of $d(x_i,x_j)$ in Eq. (\ref{discri-2}), the influence of {$x_j$} on $x_i$ is $-exp[-\gamma_1 (\Delta - 1)]$. Its first derivative for $\Delta$ is $\gamma_1 exp[-\gamma_1 (\Delta-1)] > 0$ and second one is $-\gamma_1^2 exp[-\gamma_1(\Delta-1)] < 0$. So $exp[-\gamma_1 (\Delta - 1)]$ is a monotonically increasing and concave function, and its value range is $[-1,0)$. Therefore, {$x_j$}'s influence on $x_i$ will decrease sharply when $\Delta = \Delta(x_i,x_j)$ increases and be limited to $[-1,0)$.

If $x_j$ is a close negative instance for $x_i$, the influence of {$x_j$} on $x_i$ is $\gamma_2 ln\Delta - 1$. Its first derivative is $\gamma_2 \Delta^{-1} > 0$ and second one is $-\gamma_2 \Delta^{-2} < 0$. Therefore, {$x_j$}'s influence on $x_i$ will increase sharply when $\Delta$ decreases.

\end{proof}

\begin{figure}[!t]\vspace{-0.2cm}
\centering
\includegraphics[width=0.8\linewidth, height=5.5cm]{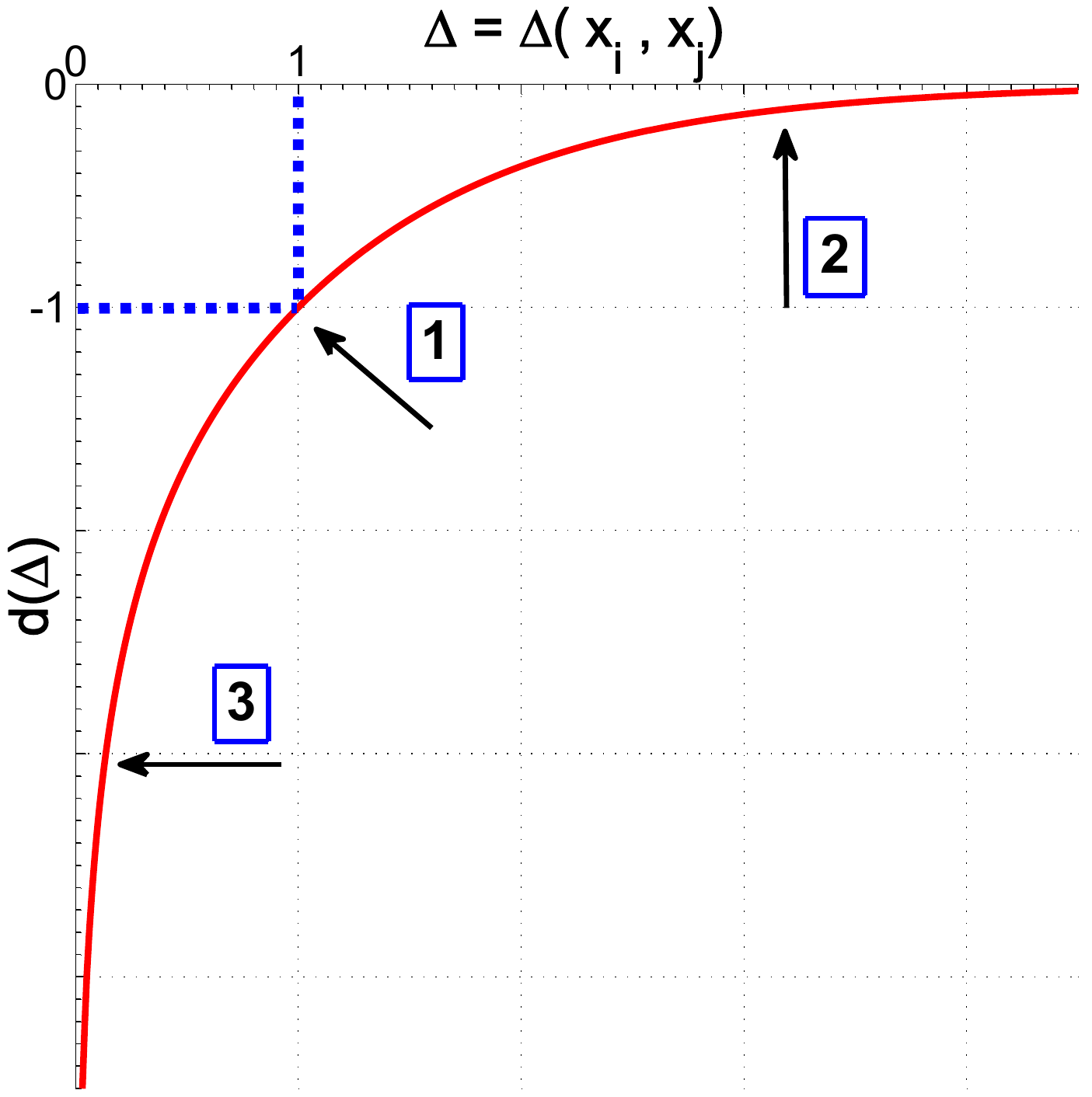} \vspace{-0.1cm}
\caption{The function image of $d(\Delta)$ that is the monotonic increasing function for $\Delta$. Suppose that $x_j$ is a far negative instance for $x_i$ if $\Delta = \Delta(x_i,x_j) > 1$ and $x_j$ is a near negative instance for $x_i$ if $\Delta < 1$. Tag '1' corresponds to the value of $d(\cdot)$ when $\Delta$ is normal. Tag '2' indicates the situation where {$d(\Delta)$}'s increase will slow down sharply when a {$x_i$}'s far negative instance is father and the limit value is 0, which means $d(\Delta)$ is robust to outliers. Tag '3' indicates the situation where $d(\Delta)$ will decrease exponentially (equal to showing a more significant influence) when a {$x_i$}'s near negative instance gets closer, which means we puts exponentially more importance on $x_i$ (more likely to regard it as a negative one) when it is closer to a near negative instances.}\vspace{-0.5cm}
\label{fig:discri}
\end{figure}

As shown in Figure \ref{fig:discri}, $d(\cdot)$ increases along with the increase of $\Delta$. The influence of close negative instance and the far one are dealt with separately. And the farther the negative instance is, the exponentially less it contributes to $d(\cdot)$. According to the definition of $\mathcal{D}$, it consists of many $d(\cdot)$s. Therefore, $\mathcal{D}$ is robust to outliers and puts exponentially more importance on near negative instances if they become closer.

\subsubsection{\textbf{Approximation of Objective Based On CSDG}}\label{pmsub2:sub2}

Different instances in PCP come from different positive bags. Finding the best labeling $\mathcal{L}$ directly for Eq. (\ref{goal-1}) is a difficult combinational optimization problem. In this section, we approximate the original goal in Eq. (\ref{goal-1}) by maximizing the total ranking score of instances in PCP based on consistent similarity and discrimination graph ($CSDG$).

\begin{definition}{\emph{(Consistent Similarity and Discrimination Graph ($CSDG$))}}\label{CSDG}
  $CSDG = (V,E)$ is an undirected weighted graph where the vertex $v_i \in V$ corresponds to the instance $x_{wb_i}$ in PCP and the edge between $v_i$ and $v_j$ is $e_{ij} \in E$ on the condition that $\mathcal{S}(x_i, x_j)>0$ and $i \neq j$. The weight for $e_{ij}$ is $E(v_{x_i},v_{x_j}) = max\{0, \mathcal{S}(x_i, x_j) + \alpha\mathcal{C}(x_i, x_j) + \beta\mathcal{D}(x_i, x_j)\}$, where $\mathcal{D}(x_i, x_j) = min\{\mathcal{D}(x_i), \mathcal{D}(x_j)\}$, and $\alpha$, $\beta$ are two balance factors.
\end{definition}

When an instance $x$ in PCP is likely to be negative, the importance of the role it plays in $CSDG$ should be undermined, i.e., decreasing the weight of edges containing $x$. This is because instances in PCP are considered as positive ones and the edges in $CSDG$ correspond to the similarity between positive instances. $\mathcal{S}(x_i, x_j)$ and $\mathcal{C}(x_i, x_j)$ represent the similarity between two instances from the global structure. $\mathcal{D}(x_i, x_j) = max\{\mathcal{D}(x_i), \mathcal{D}(x_j)\}$ indicates that the similarity should be decreased if one of two vertexes edge $e_{ij}$ contains is likely to be negative.

We approximate the optimization problem in Eq. (\ref{goal-1}) as a combination problem to maximize $E(v_{x_i},v_{x_j})$ for $CSDG$ formulated as:
\begin{equation}
\small
\label{goal-2}
\begin{split}
&\max_{\mathcal{V}} \sum_{(v_{x_i},v_{x_j}) \in CSDG}  E(v_{x_i},v_{x_j})\\
&s.t. ~\sum_{(v_{x_p},v_{x_q})} E(v_{x_p},v_{x_q}) \geqslant \sum_{(v_{x_p},v_{x_k})} E(v_{x_p},v_{x_k})\\
&~~~~~(v_{x_p},v_{x_q}) \in CSDG, \forall x_k \in (X^W \setminus  X^{CSDG})
\end{split}
\end{equation}
where $v_x$ is the vertex in $CSDG$ and corresponds to instance $x$ in PCP, $X^{W}$ is the $\mathcal{WS}s$ of all $\mathcal{WB}s$, $\mathcal{V}$ is the corresponding vertexes for $X^{W}$, and $X^{CSDG}$ represents the instances that all vertexes of $CSDG$ correspond to.

\subsubsection{\textbf{Instance Updating Strategy}}\label{pmsub2:sub3}

The intuitive way to figure out the problem in Eq. (\ref{goal-2}) is to replace the vertexes in $CSDG$ with the vertexes corresponding to the rest instances in $X^W \setminus  X^{CSDG}$ iteratively until the maximal is reached. However, it is hard and time consuming.
An approximate way is to rank vertexes in $CSDG$ and regard the vertex with the lowest ranking score as the one needed to be replaced. Then we search the most suitable substitute instance for the replaced.

\textbf{Ranking Instance in PCP:}

We propose a random walk algorithm, summarized in Algorithm \ref{alg:cal_rank}, based on PageRank \cite{page1999pagerank} to perform on $CSDG$ to rank vertexes. The intuition is that vertexes that are connected to high ranking vertexes by high weighted edges should have high ranking scores. Higher the {vertex}'s ranking score is, more positive the {vertex}'s corresponding instance is considered to be. This is because the {edge}'s weight combines the similarity among positive instances and the discrimination from negative instances. So a vertex is considered to be positive with higher probability if it is connected to more high ranking vertexes by high weighted edges shown in Figure \ref{fig:pagerank}.

\begin{algorithm}[!t]
\begin{small}
\caption{\ugg{CRS: Calculate Ranking Score}}\label{alg:cal_rank}
\begin{algorithmic}[1]
\REQUIRE ~~\\
    $CSDG = (V,E)$: Consistent similarity and discrimination graph defined in Definition \ref{CSDG} ;\\
    $d$: A damping factor;\\
    $n_{max\_ite}$: The maximal iterative number;\\
\ENSURE ~~\\
     $\mathcal{R} = {(r_{v_1}, \cdots, r_{v_M})}'$: The ranking scores for all vertexes in $CSDG$;\\

\STATE $n_{ite} \leftarrow 0$;\\
\STATE $\mathcal{R}_0$: Initialized randomly;\\
\STATE $\Upsilon_{i,i} \leftarrow w_i^T \cdot x_i$: Calculate the confidence value for each vertex;\\

\FORALL{$(i,j)\in \{ (p,q) \ | \ p,q \in (1, \cdots, M)\}$}
\STATE \textbf{if} $e_{ij}$ exists, \textbf{then}
\STATE ~~~~$E(i,j) \leftarrow$ $E(v_{x_i},v_{x_j})$
\STATE \textbf{else} $E(i,j) \leftarrow$ 0
\STATE Normalize $E(i,j)$
\ENDFOR

\WHILE{$n_{ite} \leqslant n_{max\_ite}$}
    \STATE $R_{n_{ite}+1} \leftarrow (1-d)[\Upsilon]_{(M \times 1)}+d[E]_{(M \times M)} \cdot R_{n_{ite}}$;
    \STATE $n_{ite} \leftarrow n_{ite} + 1$
\ENDWHILE
\RETURN $CSDG$;

\end{algorithmic}
\end{small}
\end{algorithm}

The vertex with the lowest ranking score is chosen to replace in each iteration phase of $CSDG$. The ranking score of each vertex is calculated iteratively by the following iteration equation:
\begin{equation}
\label{pagerank_iter}
\small
\mathcal{R}_{k+1} = (1 - d)
\begin{bmatrix}
\Upsilon_{1,1}\\
\vdots\\
\Upsilon_{M,M}
\end{bmatrix}
+d
\begin{bmatrix}
E(1,1) & \cdots & E(1,M)\\
\vdots & \ddots & \vdots\\
E(M,1) & \cdots & E(M,M)
\end{bmatrix}
\mathcal{R}_{k}
\end{equation}
where $M$ is the number of vertexes in $CSDG$, $k$ indicates the $k$th phase, $d$ is a damping factor, $\mathcal{R}_{k} = {(r_{v_1}, \cdots, r_{v_M})_k}'$ represents the ranking score of each vertex in the $k$th phase, $\Upsilon_{i,i}$ represents the self confidence value of instance $x_i$ in the process of calculating $similarity$, $E(i,j)$ is the normalized $E(v_{x_i},v_{x_j})$ and equals $E(j,i)$, and $E(i,j)$ equals $0$ if there is no edge between vertexes $v_i$ and $v_j$.

It is noteworthy that:
1) A {vertex}'s ranking score $r_{v_k}$ is determined by its confidence value $\Upsilon_{k,k}$ and the ranking scores of its adjacent vertexes. $\Upsilon_{k,k}$ represents the probability of instance $x_k$ to be classified as positive to a certain degree. The influence of its adjacent {vertexes}' ranking scores is transmitted by $E(i,j)$ which can capture the difference in relationship among vertexes.
2) The random walk algorithm is practicable. $CSDG$ is regarded as a bidirectional weighted graph without circles because the weight of edge $E(v_{x_i},v_{x_j})$ is symmetrical and there is no vertex connecting itself.
${\mathcal{R}}_{k}$ is initialized randomly. The iteration process will stop when it meets the maximal iteration number. After all vertexes get the ranking scores, the corresponding instance of the vertex with the least score will be regarded as the least positive instance.

\begin{figure}[!t]\vspace{-0.2cm}
\centering
\includegraphics[width=0.8\linewidth, height=6cm]{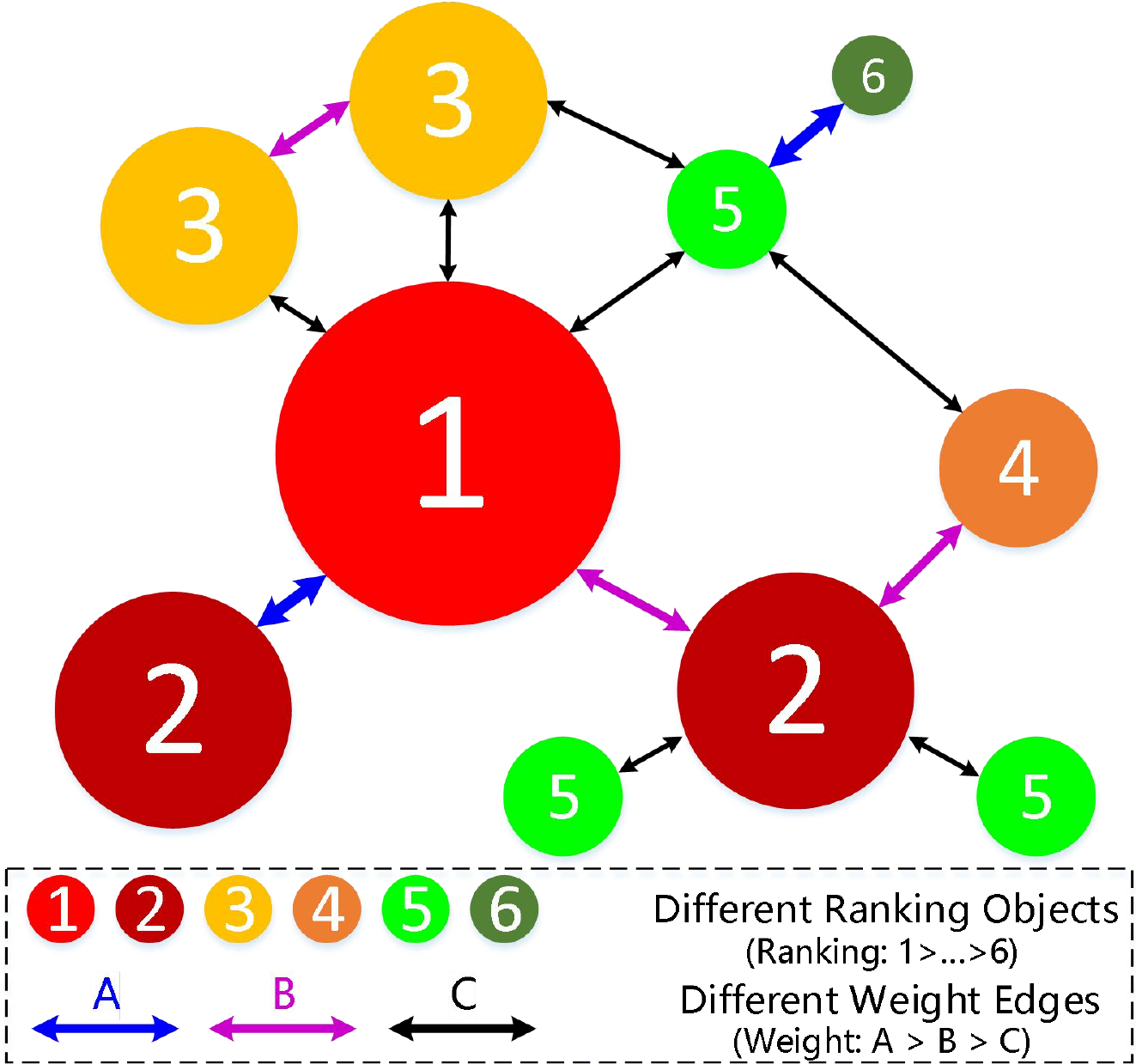} \vspace{-0.1cm}
\caption{Ten vertexes with different ranking scores (labelled with different numbers and colors) are connected by different weighted edges (labelled with different colors: A, B, and C) after the random walk algorithm employed. Higher the ranking score is, more positive the vertex is considered to be. The red circle labelled with '1' is considered to play the most important role and be the most positive one in the network because it is connected with the largest number of vertexes by relatively high weight edges. The green circle on top labelled with '5' and the dark red one on the bottom right labelled with '2' are connected to the same number of vertexes while possess different ranking scores because of different weighted edges.}\vspace{-0.5cm}
\label{fig:pagerank}
\end{figure}

\textbf{Instance Updating:}

After the least positive instance $x_t$ is discriminated, it needs to be replaced with a new one from $X^W \setminus  X^{CSDG}$.
The whole instance updating strategy is summarized in Algorithm \ref{alg:updateins}.
At first, it is intuitive to find a new one from {$x_t$}'s corresponding $\mathcal{WS}_{T}$ because there is at least one positive instance in each positive bag and should exist a more positive one in $\mathcal{WS}_{T}$ if $x_t$ is not positive enough.
Therefore, we replace $x_t$ with each instance in $\mathcal{WS}_{T}$ respectively to calculate the sum of all {vertexes}' ranking scores in $CSDG$.
There are two cases:

\textbf{\romannumeral1)} We can not find an instance in $\mathcal{WS}_{T}$ making the total score higher than that by $x_t$.
In this case, we update another vertex in $CSDG$. Specifically, {$x_t$}'s corresponding vertex in $CSDG$ is denoted as $v_t$ and the vertexes in $CSDG$ are sorted in increasing order according to the ranking score. We choose the vertex just after $v_t$ to replace. As a result, it returns to the beginning of the updating strategy. Notably, if $v_t$ is at the last of the order, the updating process will be terminated.

\textbf{\romannumeral2)} We can find an instance $x_{t'}$ in $\mathcal{WS}_{T}$ making the total score higher than that by $x_t$.
In this case, $x_{t'}$ is selected as the substitute instance for $x_{t}$. Specifically, {$x_{t'}$}'s corresponding vertex is denoted as $v_{t'}$. We rank $v_{t'}$ with the rest vertexes in $CSDG$ based on Eq. (\ref{pagerank_iter}) in increasing order. After ranked:
A) If $v_{t'}$ is at the first in the order, we choose the vertex at the second of the order to replace. Then, it returns to the beginning of the updating strategy.
B) If $v_{t'}$ is not at the first in the order, we choose the first vertex to replace. Then, it returns to the beginning of the updating strategy.
A actual updating corresponds to the process that $v_t$ is replaced by $v_{t'}$. The instance updating process will also be terminated if it reaches the maximal actual updating number.

\begin{algorithm}[!t]
\begin{small}
\caption{\ugg{IUS: Instance Updating Strategy}}\label{alg:updateins}
\begin{algorithmic}[1]
\REQUIRE ~~\\
    $CSDG$: Initialized;\\
    $n_{max\_upd}$: The maximal updating number;\\
\ENSURE ~~\\
     $CSDG$: Updated;\\

\STATE $\mathcal{R} = {(r_{v_1}, \cdots, r_{v_M})}'\leftarrow$ Invoke \textbf{CRS} (\emph{Algorithm }\ref{alg:cal_rank});\\
\STATE $v_t$: The vertex in $CSDG$ with the lowest ranking score;\\
\STATE $n_{update} \leftarrow 0$;\\
\WHILE{$n_{update} \leqslant n_{max\_upd}$}
    \STATE  $x_t$: The corresponding instance of $v_t$;
    \STATE  $x_{t'}: x_{t'} \in \mathcal{WS}_{T}$ and $v_{t'}$ corresponding to $x_{t'}$;

    \STATE \textbf{if} {$\not\exists x_{t'}$ s.t.} $(\sum_{i=1}^{M}  r_{v_i})_{(v_{t'})} > (\sum_{i=1}^{M}  r_{v_i})_{(v_{t})}$, \textbf{then}
        \STATE ~~~~$\mathcal{R}_{(v_t)}\leftarrow$ Invoke \textbf{CRS} (\emph{Algorithm }\ref{alg:cal_rank});
        \STATE ~~~~Sort vertexes in increasing order according to $\mathcal{R}_{(v_t)}$;
        \STATE ~~~~\textbf{if} $v_t$ is at the last in the order, \textbf{then}
        \STATE ~~~~~~~~\textbf{return} $CSDG$;
        \STATE ~~~~\textbf{else} $v_t \leftarrow v_{t_+}$: $v_{t_+}$ is just after $v_t$ in the order;

    \STATE \textbf{else}
        \STATE ~~~~Replace $v_t$ with $v_{t'}$ in $CSDG$;
        \STATE ~~~~$\mathcal{R}_{(v_{t'})} \leftarrow$ Invoke \textbf{CRS} (\emph{Algorithm }\ref{alg:cal_rank});
        \STATE ~~~~Sort vertexes in increasing order according to $\mathcal{R}_{(v_{t'})}$;
        \STATE ~~~~\textbf{if} {$v_{t'}$ is at the first in the order}, \textbf{then}
        \STATE ~~~~~~~~$v_t \leftarrow v_{t_{2nd}}$: $v_{t_{2nd}}$ is the second in the order;
        \STATE ~~~~\textbf{else} $v_t \leftarrow v_{t_{1st}}$: $v_{t_{1st}}$ is the first in the order;
        \STATE ~~~~Replace $v_{t'}$ with $v_t$ in $CSDG$;
        \STATE ~~~~$n_{update} \leftarrow n_{update}+1$;
\ENDWHILE
\RETURN $CSDG$;

\end{algorithmic}
\end{small}
\end{algorithm}

\subsection{Bag Classification}\label{pm:sub3}

The bag classification scheme is proposed based on the instances in updated PCP, denoted as $T^+$.
The basic idea is to embed bags into a feature space based on $T^+$ and utilize the distance between a bag and each instance in $T^+$ to represent the bag.
For bag $X_t$, the feature representation vector is:
\begin{equation}
%\small
\label{fea-bag}
Z_t = [w(X_t,x_1^{+}), w(X_t,x_2^{+}), \cdots, w(X_t,x_M^{+})]'
\end{equation}
where $x_i^{+} \in T^+$, $M$ is the number of instances in PCP, and $w(X_t,x_i^{+})$ is the distance between $X_t$ and $x_i^{+}$ based on Hausdorff distance metric as:
\begin{equation}
%\small
\label{dis-bagins}
w(X_t,x_i^{+}) = \max_{x_{tj} \in X_t} exp(-\gamma_{d} \| x_{tj} - x_i^{+} \|^2)
\end{equation}
where $\gamma_{d}$ is an empirical parameter. According to the definition of feature vector, a {bag}'s label is determined by its nearest instance to $T^+$, which means the bag is labelled positive if one of its instances is similar to any one in $T^+$. $w(X_t,x_i^{+})$ also satisfies the basic assumption of MIL that there is at least one positive instance in positive bag.

In the end, the MIL setting is transformed into the standard single instance learning problem where a classifier is trained by these vectors and their labels. A SVM classifier is employed and a new bag is classified as:
\begin{equation}
\label{decifun-bag}
L_t = sgn(G_{bag}(Z_t))
\end{equation}
where $G_{bag}(\cdot)$ is the learned decision function.
The whole algorithm procedure of PIGMIL is summarized in Algorithm \ref{alg:pigmil}.

\begin{algorithm}[!t]
\begin{small}
\caption{\ugg{PIGMIL: Positive instance detection via graph updating for multiple instance learning}}\label{alg:pigmil}
\begin{algorithmic}[1]
\REQUIRE ~~\\
    Training Set: $TR = \{((X_1^{TR}, L_1^{TR}), \cdots, ((X_{N_{TR}}^{TR}, L_{N_{TR}}^{TR}))\} \in \mathcal{X}^{d} \times \{+1,-1\}$;\\
    Test Set: $TE = \{X_1^{TE}, \cdots, X_{N_{TE}}^{TE}\} \in \mathcal{X}^{d}$;\\
    where $\mathcal{X}$ is the instance space;\\
\ENSURE ~~\\
    The labels of Test Set $\{L_1^{TE}, \cdots, L_{N_{TE}}^{TE}\} \in \{+1,-1\}$;\\
\item[] \textbf{// Initialization (Section \ref{pm:sub1}):}
\STATE $\mathcal{WS}_j \leftarrow$ $\{(x_{jk_1}, \cdots, x_{jk_{n_{ws}}}) \ | \ f_{KDE_{min}}(x_{ji}) \leqslant T_{ws_j}, x_{ji} \in X_j^{TR}, L_j^{TR} = +1 \}$, where $n_{ws}$ is the size of $\mathcal{WS}_j$;\\
\STATE $\mathcal{WB}_j \leftarrow$ $\{ X_j^{TR} | t_j \geqslant h_{T_{test}} \}$, where $t_j$ is the t-value for $X_j^{TR}$ and $h_{T_{test}}$ is a threshold value;\\
\STATE PCP $\leftarrow$ $\{ (x_{p_1}^*, \cdots, x_{p_{n_{wb}}}^*) \ | \ x_{p_j}^*$ = $\arg\min\limits_{x \in \mathcal{WS}_{p_j}}$ $f_{KDE_{min}}(x) \}$, where $X_{P_j}^{TR}$ is a $\mathcal{WB}$ and $\mathcal{WS}_{p_j}$ is its $working \ set$;\\
\item[] \textbf{// PCP Updating (Section \ref{pm:sub2}):}
\STATE Initialized $CSDG$ $\leftarrow$ Construct $CSDG$ based on PCP: $\mathcal{S}(x_i, x_j)$, $\mathcal{C}(x_i, x_j)$, and $\mathcal{D}(x_i, x_j)$;
\STATE Updated $SCDG$ (updated PCP) $\leftarrow$ Invoke $\textbf{IUS}$ $(\emph{Algorithm}$ \ref{alg:updateins});
\item[] \textbf{// Bag Classification (Section \ref{pm:sub3}):}
\STATE $L_t^{TE}$ $\leftarrow$ $sgn(G_{bag}(Z_t^{TE}))$: transform $X_t^{TR}$, $X_t^{TE}$ into $Z_t^{TR}$, $Z_t^{TE}$ respectively, and employ a SVM classifier learned by $(Z_t^{TR}, L_t^{TR})$ to classify $Z_t^{TE}$;
\RETURN $L_t^{TE}$;

\end{algorithmic}
\end{small}
\end{algorithm}

\section{Experiments}\label{ex}

\subsection{Data Sets, Baseline Methods, and Experimental Settings}\label{ex:sub1}

Three synthetic MIL data sets (BASIC, RHOMBUS, RING shown in Figure \ref{fig:synthetic}(a)- \ref{fig:synthetic}(c) separatively) are constructed to verify {PIGMIL}'s ability to detect TPIs. Each data set contains 20 positive and 20 negative bags, each positive bag contains 4 positive and 4 negative instances, and each negative bag contains 8 negative instances. Negative and positive instances are generated from uniform distribution and normal distribution respectively.
Three kinds of real-world MIL data sets are utilized to verify {PIGMIL}'s classification accuracy compared to classic MIL methods:
1) Musk-1 and Musk-2 \cite{andrews2002support}.
2) Elephant, Fox and Tiger \cite{andrews2002support}.
3) UCSB Breast \cite{kandemir2014empowering} belongs to image classification and is used in tissue microarray (TMA) based diagnosis in malignant breast cancer. Each image (bag) is split into equal-sized grid (instance) and its goal is to determine an image as benign or malignant.

To evaluate {PIGMIL}'s ability to detect TPIs on synthetic data sets and its performance on real-world data sets, some baseline methods are implemented:

\vspace{0.1cm}\noindent\textit{\uppercase\expandafter{\romannumeral1}: The TPI based methods}

\begin{itemize}
\small
    \item[1.] $APR$: The first method designed for MIL problem constructs a rectangle that is parallel with axis and tries to cover positive instances as many as possible \cite{dietterich1997solving}.
    \item[2.] $DD$: Recognize instances with the highest DD value, and regard these instance as TPIs \cite{zhang2001dd}.
    \item[3.] $MILD$: Utilize the ambiguous information of instances in the positive bags to distinguish the true positive instances with two feature representation \cite{li2010mild}.
    \item[4.] $mi$-$Sim$: Learn the similarity between instances in positive bags combined with the {similarity}'s consistency \cite{rastegari2015discriminative}.
    \item[5.] $KDE$$_{min}$: (For TPI detection) The instance with the lowest $f_{KDE_{min}}(x)$ (defined in Equation (\ref{KDEmin})) of each positive bag makes up TPIs.
    \item[6.] $KDE$\footnote[1]{\tgg{$f_{KDE}(x) = {(Z N^-)}^{-1} \sum_{L_j=-1} \sum_{x_{ji} \in X_j} exp(-\gamma \|x - x_{ji}\|)$}}: (For TPI detection) The instance with the lowest $f_{KDE}(x)$ of each positive bag makes up TPIs \cite{duda2012pattern}.
    \item[7.] $KDE$$_{max}$\footnote[2]{\tgg{$f_{KDE_{max}}(x) = {(Z N^-)}^{-1} \sum_{L_j=-1} \max_{x_{ji} \in X_j} exp(-\gamma \|x - x_{ji}\|)$}}: (For TPI detection) The instance with the lowest $f_{KDE_{max}}(x)$ of each positive bag makes up TPIs.
\end{itemize}

\vspace{0.1cm}\noindent\textit{\uppercase\expandafter{\romannumeral2}: The non-TPI based methods}
\begin{itemize}
\small
    \item[1.] $Citation \ kNN$: Apply k-nearest neighbor method into MIL and define bag-level distance between bags based on the minimum Hausdorff distance \cite{wang2000solving}.
    \item[2.] $MI$-$Kernel$: Apply the set kernel method to bags represented by sets of feature vectors \cite{gartner2002multi}.
    \item[3.] $MILES$: Try to discriminate target instances and measure similarity between bags according to their closeness to target instances \cite{chen2006miles}.
    \item[4.] $miGraph$: Suppose that instances in a bag are non-iid and takes advantage of graph kernel~\cite{zhou2009multi}.
    \item[5.] $Clustering$ $MIL$: Construct a 'concept' (a spherical area) by clustering all positive instances and instances located in the concept are labelled as positive \cite{tax2010detection}.
    \item[6.] $MInD$(Hausdorff): A MIL framework that takes the Hausdorff distance to measure difference between bags \cite{cheplygina2015multiple}.
\end{itemize}
%MILIS \cite{fu2011milis} updates the prototypes based on a SVM model to construct the feature representation for each bag and employs a SVM to classify a new bag.
%mi-SVM \cite{andrews2002support} searches for a hyperplane where each positive bag has at least one instance located in the positive space while all negative {bags}'s instances are in the negative one.

%\vspace{-0.25cm}
%\noindent\rule{8.75cm}{0.05em}
%\vspace{-0.07cm}

\vspace{0.1cm}

All reported results are based on 5 times 10-fold cross-validation. All data features are normalized so that each feature shares zero mean and unit variance.
A linear kernel SVM is selected as the classifier.
$\alpha$ and $\beta$ (in Equation \ref{goal-1}) are set to be $max(10, log\frac{C(x_i,x_j)}{S(x_i,x_j)})$ and $max(10,log\frac{\alpha \cdot S + C}{D})$.
We set the size of $\mathcal{WS}$ as 40\% of a positive bag, i.e., $T_{ws_{j}}$ is set to be the least fortieth $ws(\cdot)$.
The quantile when selecting $\mathcal{WB}$ is set to be 1.5 which represents the 90\% confidence level
$\gamma_1$, $\gamma_2$ (Equation (\ref{discri-1})), $Z$ (Equation (\ref{discri-2})), and $\gamma_d$ (Equation (\ref{dis-bagins})) are set to be 1.
$d$ (Equation (\ref{pagerank_iter})) is set to be 0.8.
$n_{max\_ite}$ (Algorithm \ref{alg:cal_rank}) and $n_{max\_upd}$ (Algorithm \ref{alg:updateins}) are 10 and 20 respectively.
Moreover, all methods are executed on an Intel Core 2 Duo CPU (2.10GHz) PC.

Because the specific label of individual instance in positive bags for real-world data sets is unknown while known for the synthetic ones, we test {PIGMIL}'s ability of TPI detection on the synthetic data sets and compare it with those of some baseline methods.

\subsection{TPI Detection Comparison on Synthetic Data Sets}\label{ex:sub2}

\begin{figure*}[!t]
 \centering
 \subfigure[BASIC]{\includegraphics[scale =0.25]{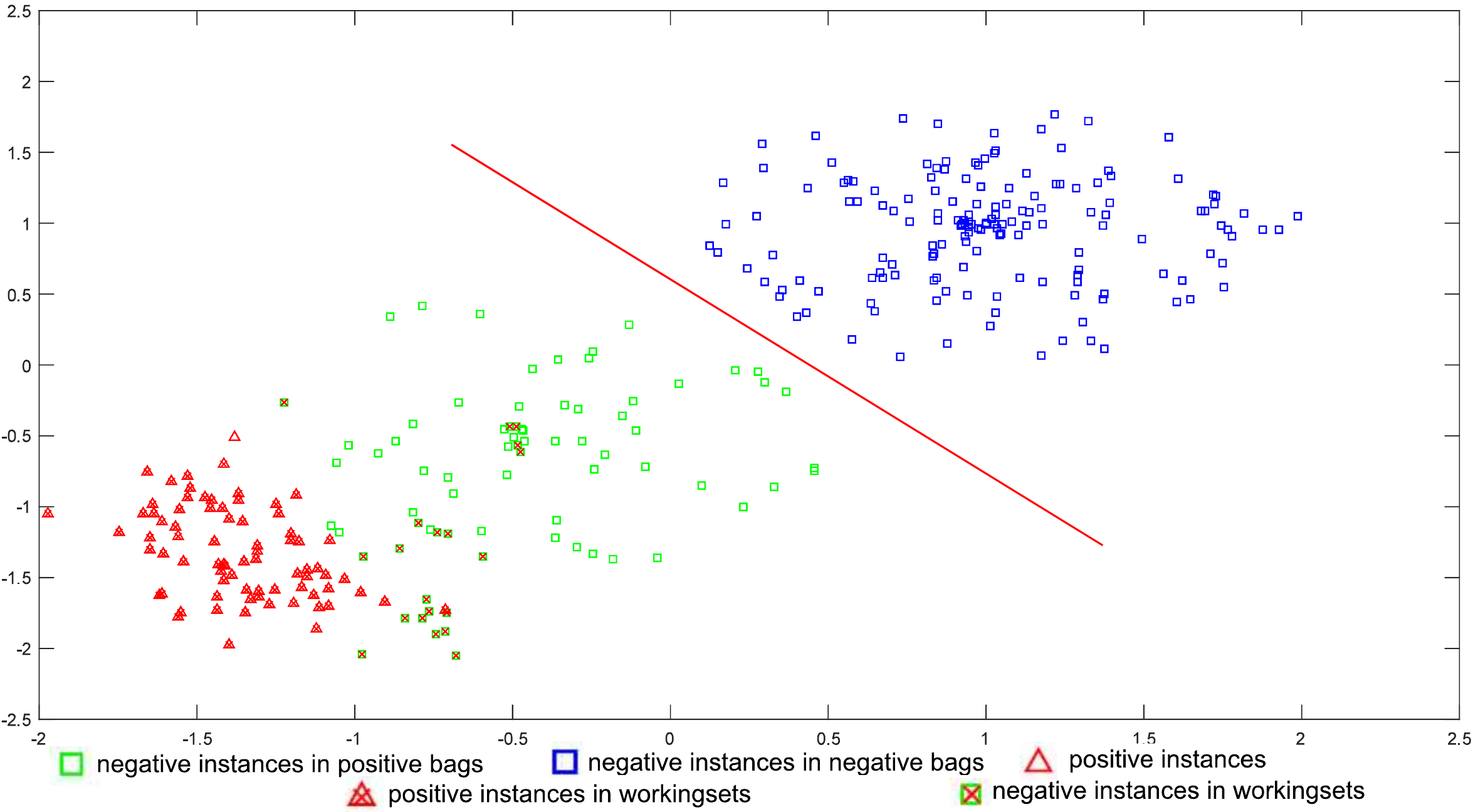}}~\vspace{0.1cm}
 \subfigure[RHOMBUS]{\includegraphics[scale =0.25]{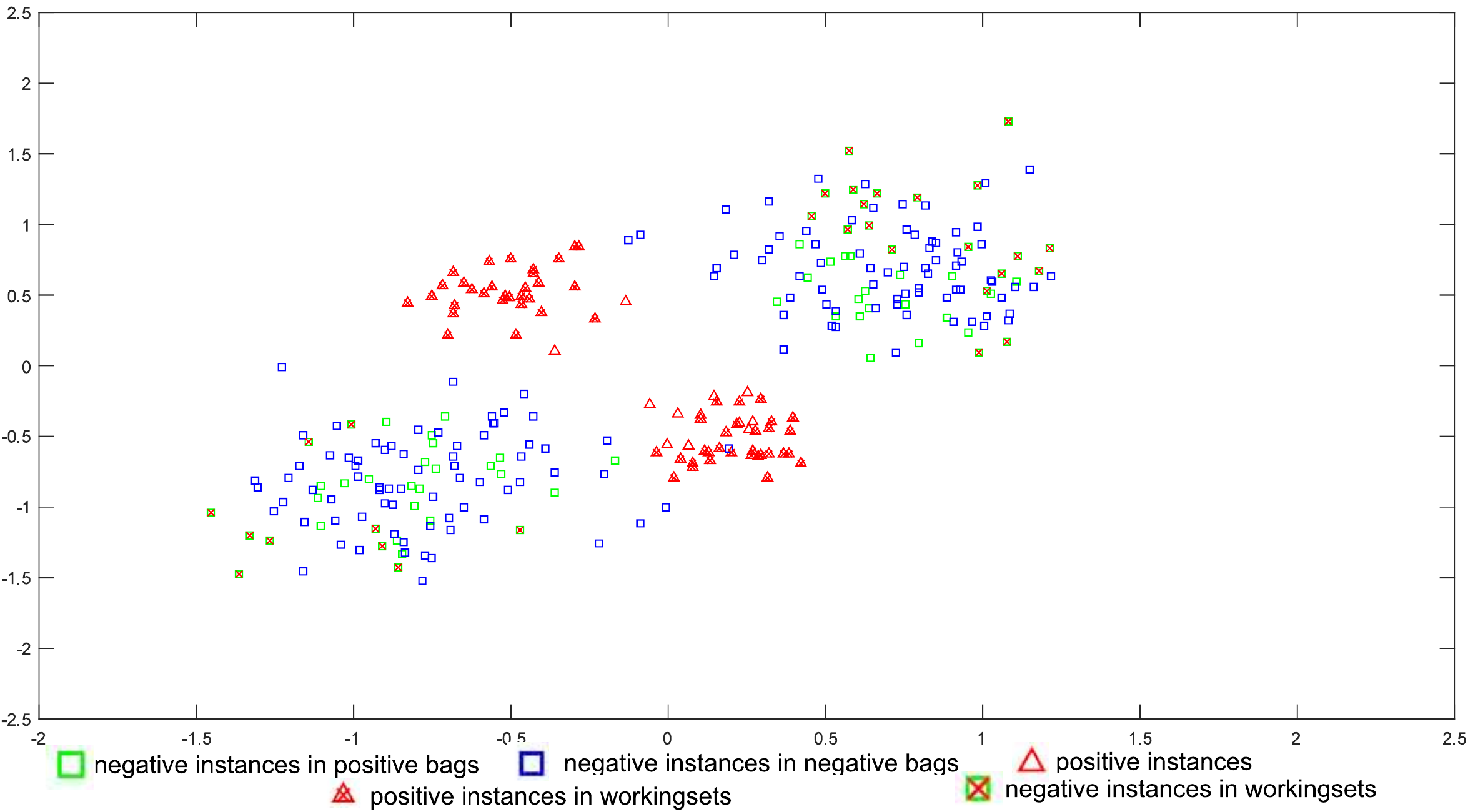}}~\vspace{-0.2cm}
 \subfigure[RING]{\includegraphics[scale =0.25]{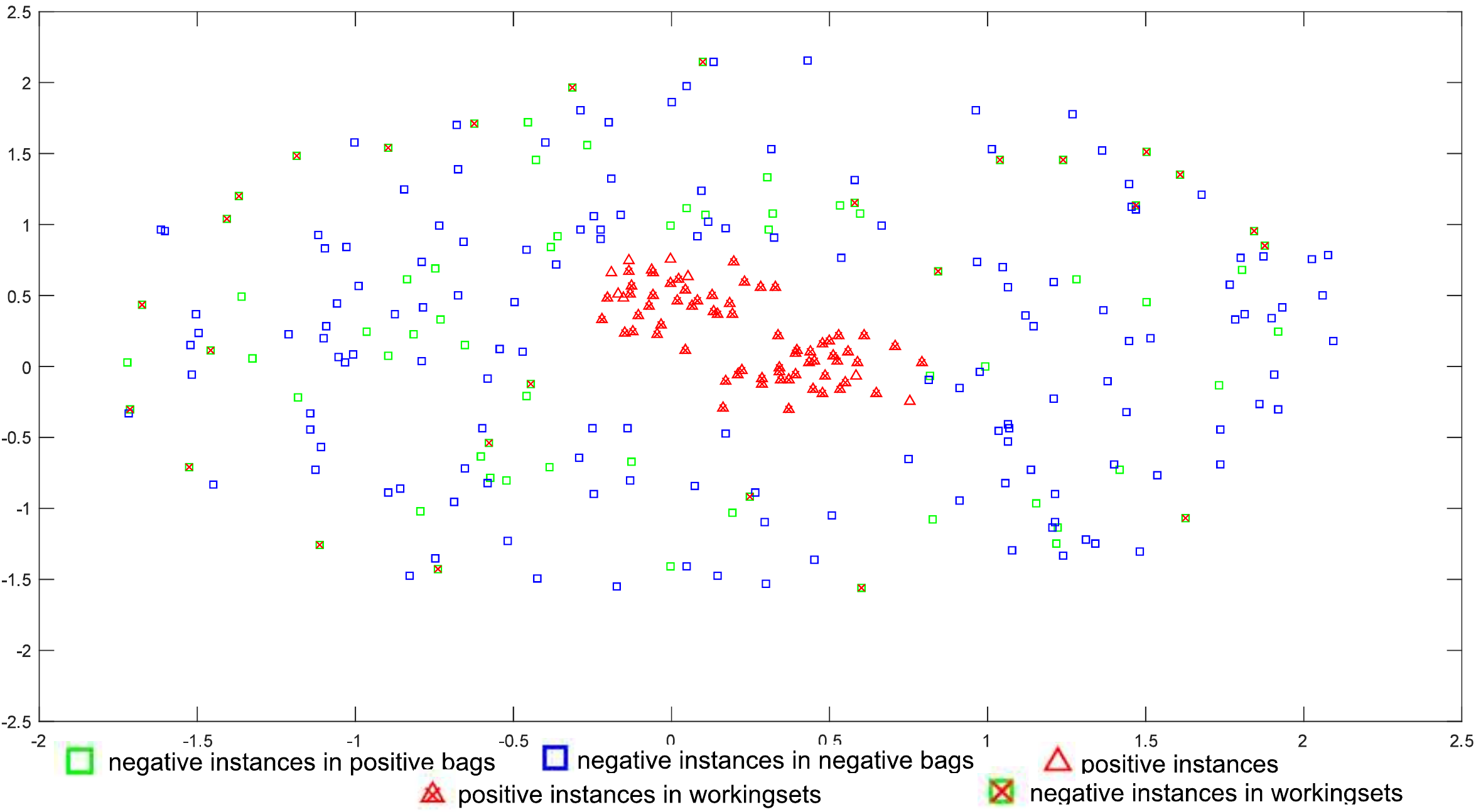}}~\vspace{-0.2cm}
 \caption{Three synthetic MIL data sets:
 1) BASIC is linearly separable and the negative instances in negative bags arise from the same uniform distribution. Negative and positive instances in positive bags arise from another uniform distribution and a normal distribution respectively.
 2) RHOMBUS is linearly inseparable. The positive and negative instances arise from two normal distributions and two uniform distributions respectively. Its negative instances in positive bags are randomly selected from the negative instance set.
 3) RING is linearly inseparable. Positive instances arise from a normal distribution located at the center, and negative ones arise from a uniform distribution located at the area between two concentric circles.
 Instances labelled with red '$\times$' construct $working \ {set}s$ ($\mathcal{WS}s$).
 The size of $\mathcal{WS}$ is 60\% of a positive {bag}'s size. The critical value of $t$ involving in selecting $working \ bags$ ($\mathcal{WB}s$) is set to be 1.5.}
 \label{fig:synthetic}\vspace{-0.3cm}
\end{figure*}

\begin{figure*}[!t]
 \centering
 \subfigure[Result on BASIC]{\includegraphics[scale =0.25]{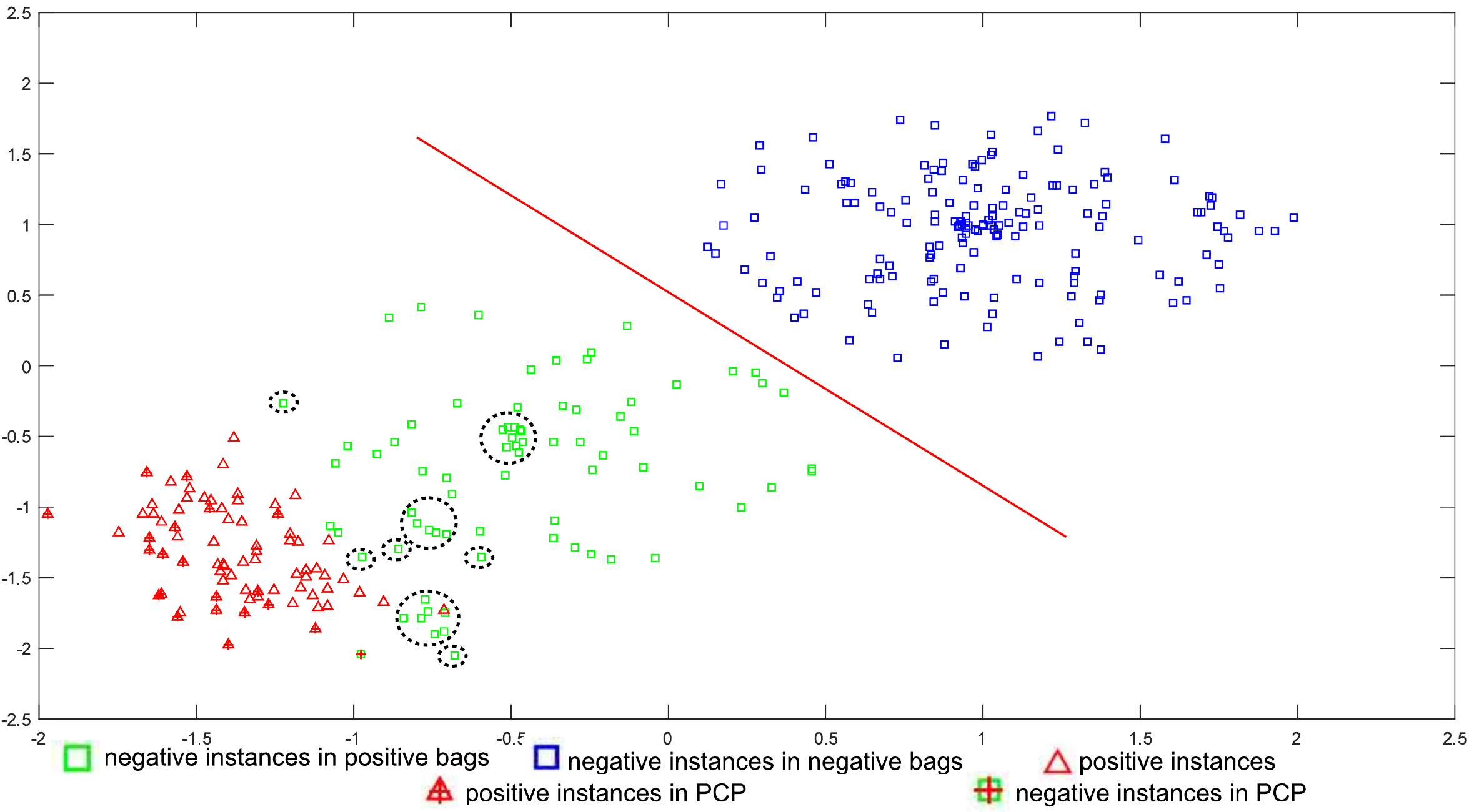}}~\vspace{0.1cm}
 \subfigure[Result on RHOMBUS]{\includegraphics[scale =0.25]{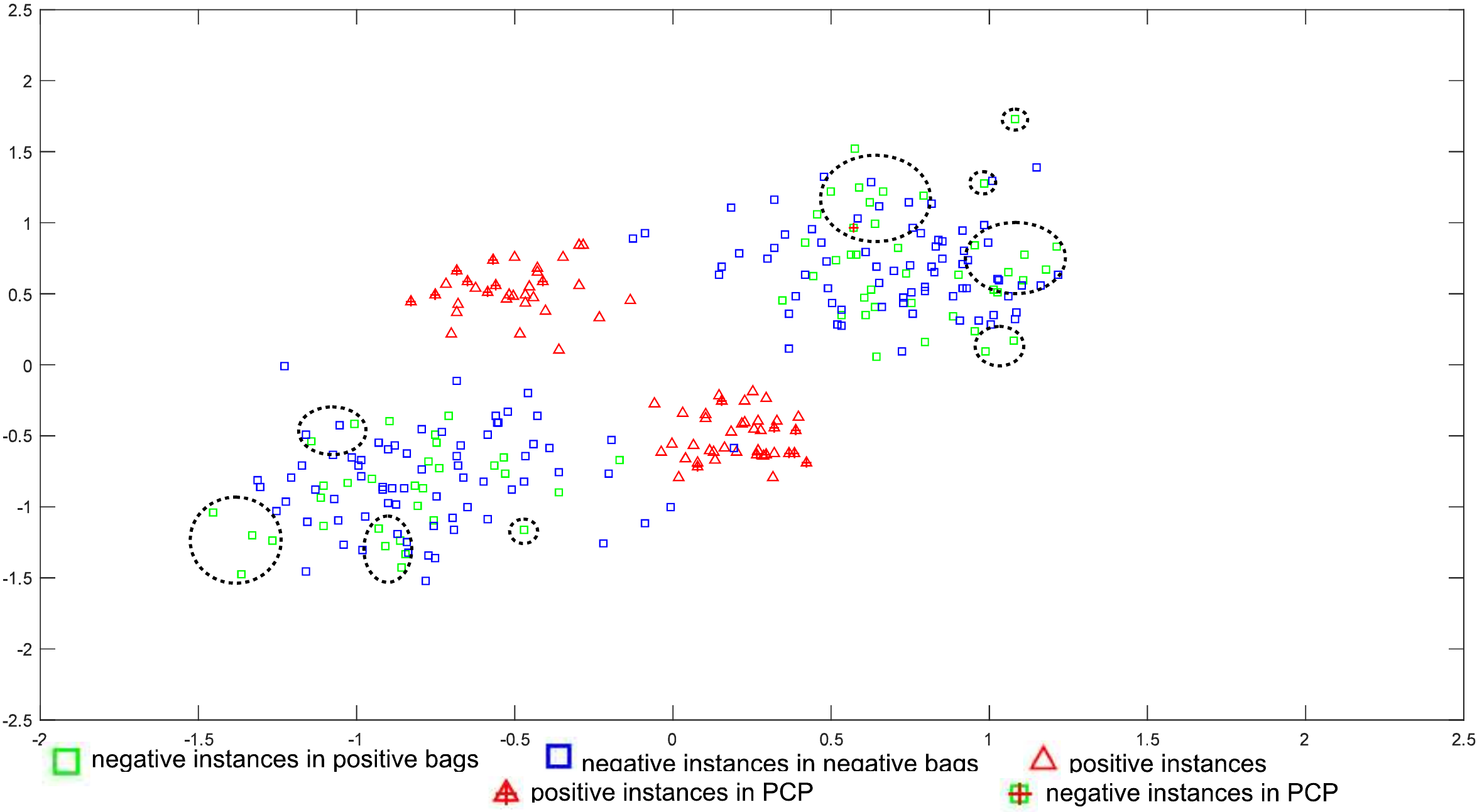}}~\vspace{-0.2cm}
 \subfigure[Result on RING]{\includegraphics[scale =0.25]{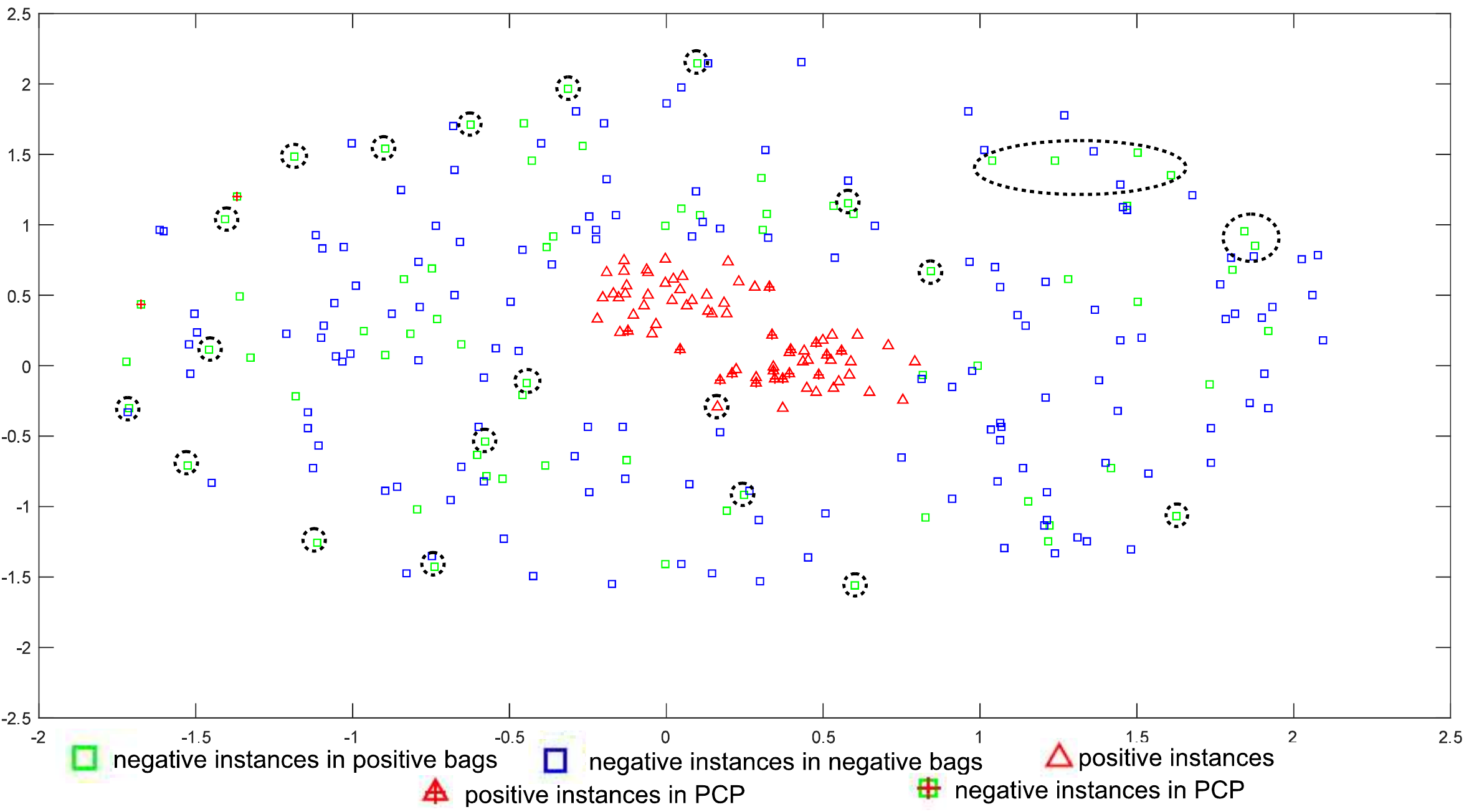}}~\vspace{-0.2cm}
 \caption{{PIGMIL}'s detection of TPIs on \textit{BASIC}, \textit{RHOMBUS}, and \textit{RING}.
 Instances labelled with red '+' construct PCP.
 Black dashed circles correspond to the negative instances that are not in PCP while were in $\mathcal{WS}$s (shown in Figure \ref{fig:synthetic}).}
 \label{fig:syn_TPI}\vspace{-0.3cm}
\end{figure*}

Figure \ref{fig:synthetic} presents the detail of BASIC, RHOMBUS, and RING.
Figure \ref{fig:syn_TPI} reports the performance of {PIGMIL]}'s detection of TPIs on the data sets.
According to Figures \ref{fig:synthetic} (a), (b), and (c), $\mathcal{WS}s$ contained most positive instances in positive bags, which verified the validity of $\mathcal{WS}$s.
According to Figures \ref{fig:syn_TPI} (a), (b), and (c), there were many negative instances (encompassed by black dashed circles) in PCP but had been not in $\mathcal{WS}s$ of Figure \ref{fig:synthetic}. The ratio of positive instances to negative ones in PCP was higher than that in $\mathcal{WS}s$ of Figure \ref{fig:synthetic}, which demonstrated {PIGMIL}'s great ability to detect TPIs.

Table \ref{tab:TPIacc_com} reports the comparison of TPI detection for PIGMIL, DD, MILD, mi-Sim, KDE$_{min}$, KDE, and KDE$_{max}$.
Overall, PIGMIL achieved the best performance on the data sets, which demonstrated {PIGMIL}'s flexibility to different shape of data sets.
DD and MILD showed a low performance on BASIC mainly because positive and negative instances are close.
The bad performance of KDE and KDE$_{max}$ on RING mainly because TPIs are too centralized while negative instances are too dispersive.

\begin{table}[!t]
\addtolength{\tabcolsep}{-5pt}
\renewcommand{\arraystretch}{1.25}
\centering
\caption{Comparison of TPI detection accuracy (\%) with the average one on BASIC, RHOMBUS, and RING. The highest accuracy for each data set is in bold.}\vspace{0cm}
\label{tab:TPIacc_com}
\begin{tabular}{cccccccc}
\toprule
\backslashbox{Data Set}{Method}     &        PIGMIL     &    DD  &   MILD  & mi-Sim    &  KDE$_{min}$  & KDE       & KDE$_{max}$ \\
\midrule
BASIC               &   \textbf{95.0}   &	30.0  &	0.0    & 60.0      &  80.0         &   80.0    & 80.0  \\
RHOMBUS             &   \textbf{95.0}   &	\textbf{95.0} &	\textbf{95.0}   & 75.0      &  85.0         &  90.0     & \textbf{95.0}  \\
RING                &   \textbf{100}    &	95.0  &	\textbf{100}    & 95.0      &  \textbf{100}          &   0.0     & 0.0     \\
\midrule
Average Accuracy    &   \textbf{96.7}    &	73.3  &	65    & 76.7      & 88.3          &   56.7     & 58.3   \\
\bottomrule
\end{tabular}
\end{table}

\subsection{Accuracy Comparison on real-world data sets}\label{ex:sub3}

\begin{table*}[!t]
\scriptsize
\addtolength{\tabcolsep}{-3pt}
\renewcommand{\arraystretch}{1.5}
\centering
\caption{Comparison of TPI detection accuracy (\%) with the average one on BASIC, RHOMBUS, and RING. The highest accuracy for each data set is in bold.}\vspace{0cm}
\label{tab:bagacc_com}
\begin{tabular}{ccccccccccc}
\toprule
\backslashbox{Data Set}{Method} & PIGMIL & APR         & MILD          & mi-Sim        & Citation kNN     & MI-Kernel      & MILES          &   miGraph       & Clustering MIL & MInD(Hausdorff)\\
\midrule
Musk-1      &\textbf{83.4 $\pm$ 12.0}& 76.9 $\pm$ 13.8 &   79.5 $\pm$ 13.54 &  82.4 $\pm$ 12.8 & 82.7 $\pm$ 14.8 & 54.8 $\pm$ 14.4& 72.0 $\pm$ 13.7&   \textbf{85.5 $\pm$ 12.6 $\bullet$}& 65.6 $\pm$ 15.7& 48.4 $\pm$ 13.8  \\
Musk-2      &\textbf{87.2 $\pm$ 9.9}& 74.4 $\pm$ 14.1&   75.6 $\pm$ 16.4&      75.2 $\pm$ 18.2& 83.1 $\pm$ 11.1& 77.3 $\pm$ 16.5& \textbf{88.2 $\pm$ 10.8 $\bullet$ }&   72.9 $\pm$ 13.9& 59.0 $\pm$ 13.0& 75.6 $\pm$ 21.8  \\
Elephant    &\textbf{80.0 $\pm$ 8.8}& 75.2 $\pm$ 8.4 &   79.4 $\pm$ 9.9 & 78.3 $\pm$ 8.6& 76.1 $\pm$  8.7& 67.1 $\pm$  8.8& \textbf{82.0 $\pm$ 7.3 $\bullet$ }&   78.8 $\pm$ 8.3&  71.4 $\pm$ 10.3& 54.0 $\pm$ 4.3  \\
Fox        & \textbf{58.5 $\pm$ 9.1}& 55.7 $\pm$ 10.6  &   57.6 $\pm$ 12.0&      52.8 $\pm$ 9.2& 58.3 $\pm$ 12.3& 56.5 $\pm$  7.8& \textbf{63.8 $\pm$ 10.6 $\bullet$}&   53.3 $\pm$ 10.1& 53.9 $\pm$  9.5& 58.0 $\pm$ 9.8  \\
Tiger       &\textbf{79.0 $\pm$ 8.8 $\bullet$} & 59.8 $\pm$ 9.7 & 73.5 $\pm$ 9.6& \textbf{75.5 $\pm$ 9.1}& 69.8 $\pm$ 10.6& 65.9 $\pm$  4.9& \textbf{74.5 $\pm$  8.4}&   74.2 $\pm$ 9.5&  57.1 $\pm$ 11.6& 54.8 $\pm$ 4.2  \\
UCSB Breast &\textbf{61.6 $\pm$ 9.6}& 50.5 $\pm$ 12.8 & 48.4 $\pm$ 14.6& 57.5 $\pm$ 19.5 & \textbf{69.1 $\pm$ 20.6 $\bullet$}& 55.5 $\pm$  7.1& 55.4 $\pm$  6.7&   50.5 $\pm$ 21.5& 57.4 $\pm$ 21.5& 55.6 $\pm$ 7.2  \\

Eastwest &\textbf{61.0 $\pm$ 21.7}& \textbf{75.0 $\pm$ 23.9 $\bullet$} & 44.0 $\pm$ 27.9 & 59.0 $\pm$ 22.7 & 59.7 $\pm$ 29.8 & 48.7 $\pm$  23.3 & 49.6 $\pm$  26.5  &  57.5 $\pm$ 25.0 & 61.0 $\pm$ 20.9 & 54.0 $\pm$ 20.3  \\
Westeast & 52.0 $\pm$ 19.0& 40.0 $\pm$ 17.7 & 44.0 $\pm$ 24.0 & 55.0 $\pm$ 23.9 & \textbf{61.4 $\pm$ 25.5$\bullet$}& 51.6 $\pm$  25.8 & 44.1 $\pm$  25.4 &   50.7 $\pm$ 23.9 & xx $\pm$ yy & \textbf{59.8 $\pm$ 14.8}  \\
Atom &\textbf{80.8 $\pm$ 8.2}& 66.5 $\pm$ 0.4 & 64.7 $\pm$ 9.3 & 72.4 $\pm$ 5.5 & 78.3 $\pm$ 9.4 & \textbf{84.5 $\pm$  10.4 $\bullet$}& 65.2 $\pm$  10.0 &  80.4 $\pm$ 8.8 & 58.6 $\pm$ 13.9 & 64.0 $\pm$ 9.3  \\
Bond &\textbf{81.5 $\pm$ 7.5 $\bullet$}& 66.6 $\pm$ 2.6 & 66.6 $\pm$ 12.4 & xx $\pm$ yy & \textbf{79.0 $\pm$ 10.3} & 72.2 $\pm$  6.4 & 62.2 $\pm$  25.6 &   78.7 $\pm$ 8.6 & 61.1 $\pm$ 12.3 &  69.2 $\pm$ 14.6  \\
Chain & 77.1 $\pm$ 9.5& 66.4 $\pm$ 1.1 & 65.2 $\pm$ 10.0 & xx $\pm$ yy & 71.0 $\pm$ 9.5 & \textbf{84.8 $\pm$  6.6 $\bullet$} & 66.3 $\pm$  10.9 &   \textbf{84.7  $\pm$ 7.1} & 66.4 $\pm$  10.5 & 72.1 $\pm$ 15.7  \\
\bottomrule
\end{tabular}
\end{table*}

We choose Musk-1, Musk-2, Elephant, Fox, Tiger, and UCSB to test {PIGMIL}'s classification accuracy on real-world compared to APR, MILD, mi-Sim, Citation kNN, MI-Kernel, MILES, miGraph, Clustering MIL, and MInD(Hausdorff).
The accuracy comparison is reported in Table \ref{tab:bagacc_com} where PIGMIL achieved a competitive performance. PIGMIL outperformed most other methods, especially on Elephant, Fox, and Tiger possibly because TPIs of these data sets (regions of interest, ROIs) are easier to discriminate compared to others, like a specific drug molecule shape in Musk-1 and Musk-2.

\section{Discussion}\label{dis}

\subsection{Sensitivity to Global Similarity ($\mathcal{S}$+$\mathcal{C}$) and Robust Discrimination ($\mathcal{D}$)}\label{dis:sub1}

PIGMIL can capture the global similarity ($\mathcal{S}$+$\mathcal{C}$) of TPIs and their robust discrimination ($\mathcal{D}$) from negative instances. To measure the influence of $\mathcal{S}$+$\mathcal{C}$ and $\mathcal{D}$ on TPI detection accuracy, we changed the ratio of $\mathcal{D}$ to $\mathcal{S}$+$\mathcal{C}$ (scaling $\mathcal{D}$ to different levels).
Figures \ref{fig:SC_D} (a), (b), and (c) present the change of TPI detection accuracy with different ratio of $\mathcal{S}$+$\mathcal{C}$ and $\mathcal{D}$ on BASIC, RHOMBUS, and RING separately.

According to Figure \ref{fig:SC_D} (a), the accuracy increased when the ratio became bigger, which indicated that $\mathcal{D}$ contributed more to the accuracy than $\mathcal{S}$+$\mathcal{C}$ on this kind of data set.
In Figure \ref{fig:SC_D} (b), the highest accuracy was reached when $\mathcal{S}$+$\mathcal{C}$ and $\mathcal{D}$ were at the same order of magnitude. This was mainly because TPIs or negative instances were symmetrical so that $\mathcal{S}$+$\mathcal{C}$ and $\mathcal{D}$ played the similar important roles.
Figure \ref{fig:SC_D} (c) indicates that {$\mathcal{D}$}'s increase contributed to the increase of accuracy while the contribution was limited.

\subsection{Sensitivity to Noise}\label{dis:sub2}

\begin{figure}[!t]
 \centering
 \subfigure[BASIC]{\includegraphics[scale =0.13]{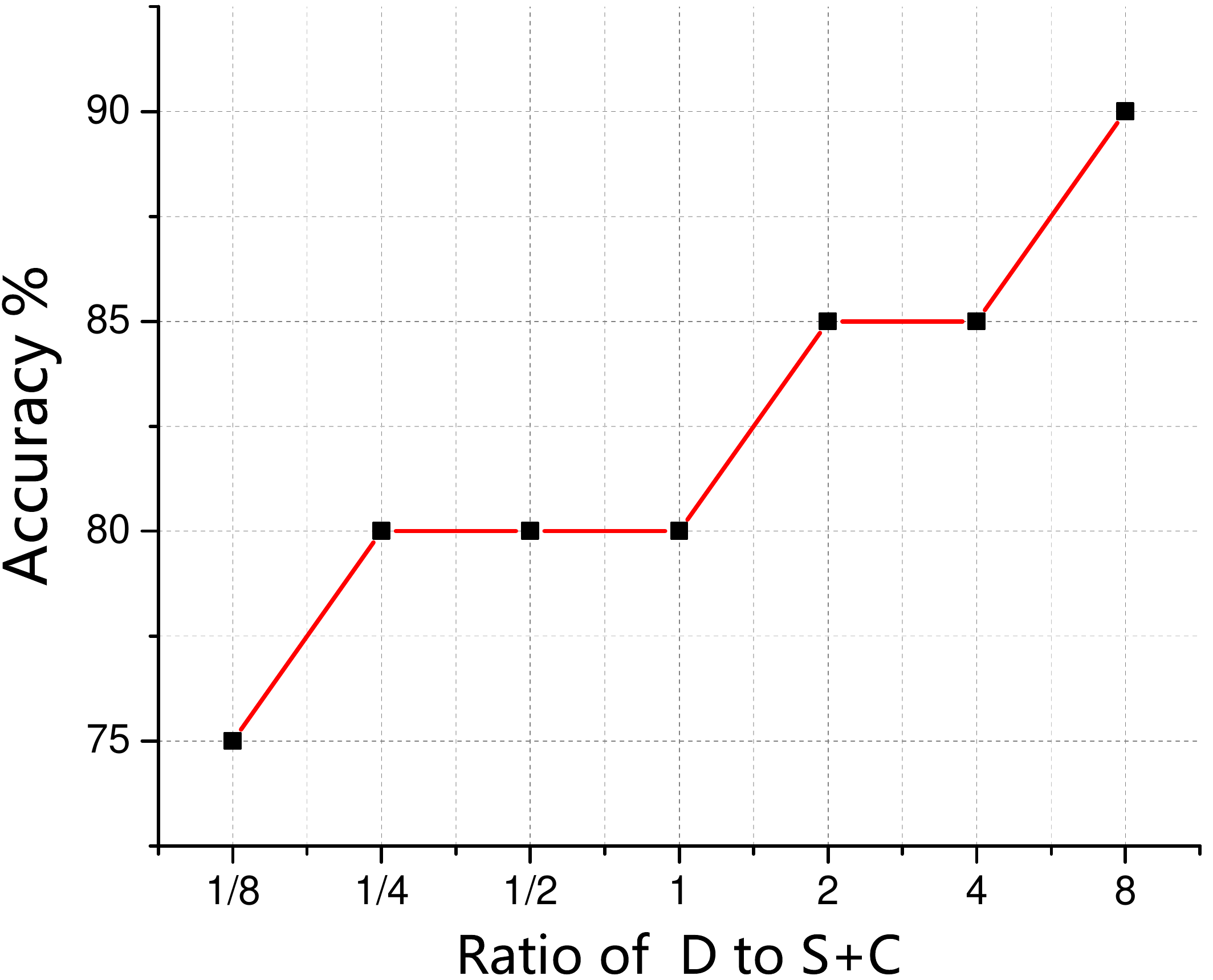}}~\vspace{0.1cm}
 \subfigure[RHOMBUS]{\includegraphics[scale =0.13]{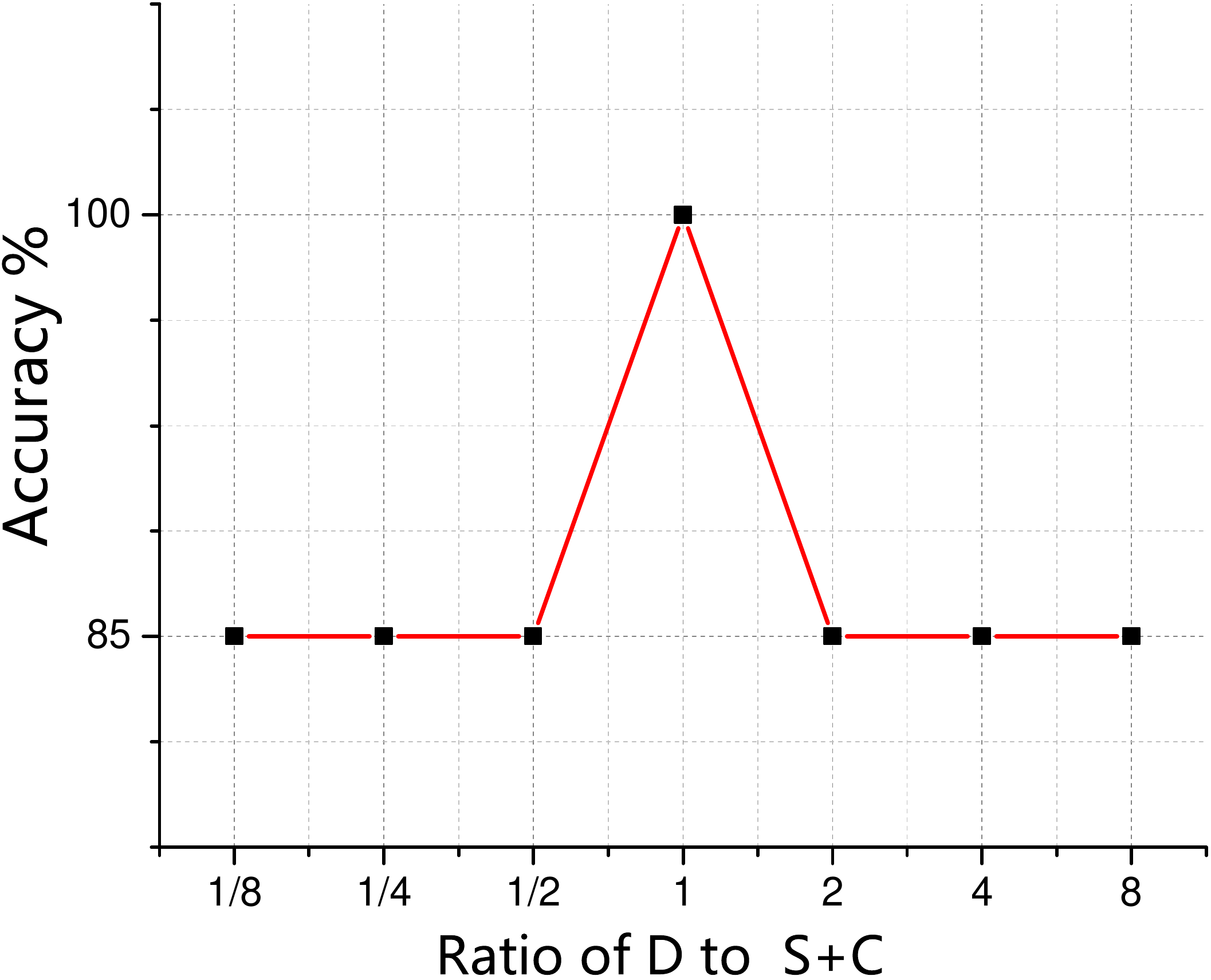}}~\vspace{-0.2cm}
 \subfigure[RING]{\includegraphics[scale =0.13]{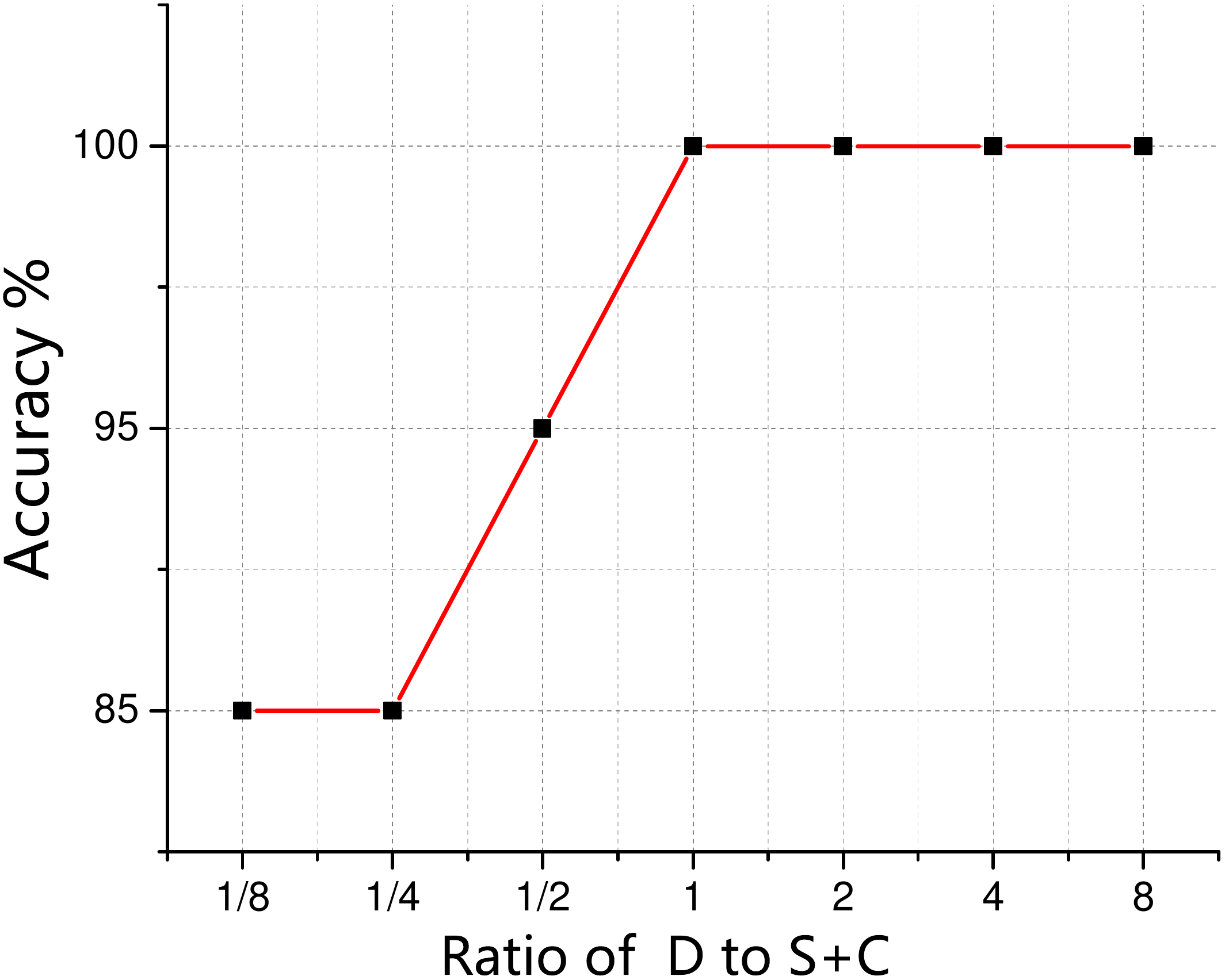}}~\vspace{-0.2cm}
 \caption{TPI detection accuracy of PIGMIL with different ratios of the global similarity ($\mathcal{S}$+$\mathcal{C}$) of TPIs and the robust discrimination ($\mathcal{D}$) on \textit{BASIC}, \textit{RHOMBUS}, and \textit{RING}. Specifically, the ratio of '2' indicates that $\mathcal{D}$ is scaled to the twice order of magnitude of $\mathcal{S}$+$\mathcal{C}$.}
 \label{fig:SC_D}\vspace{-0.3cm}
\end{figure}

\begin{figure}[!t]
 \centering
 \subfigure[Sensitivity to Noise]{\includegraphics[scale =0.19]{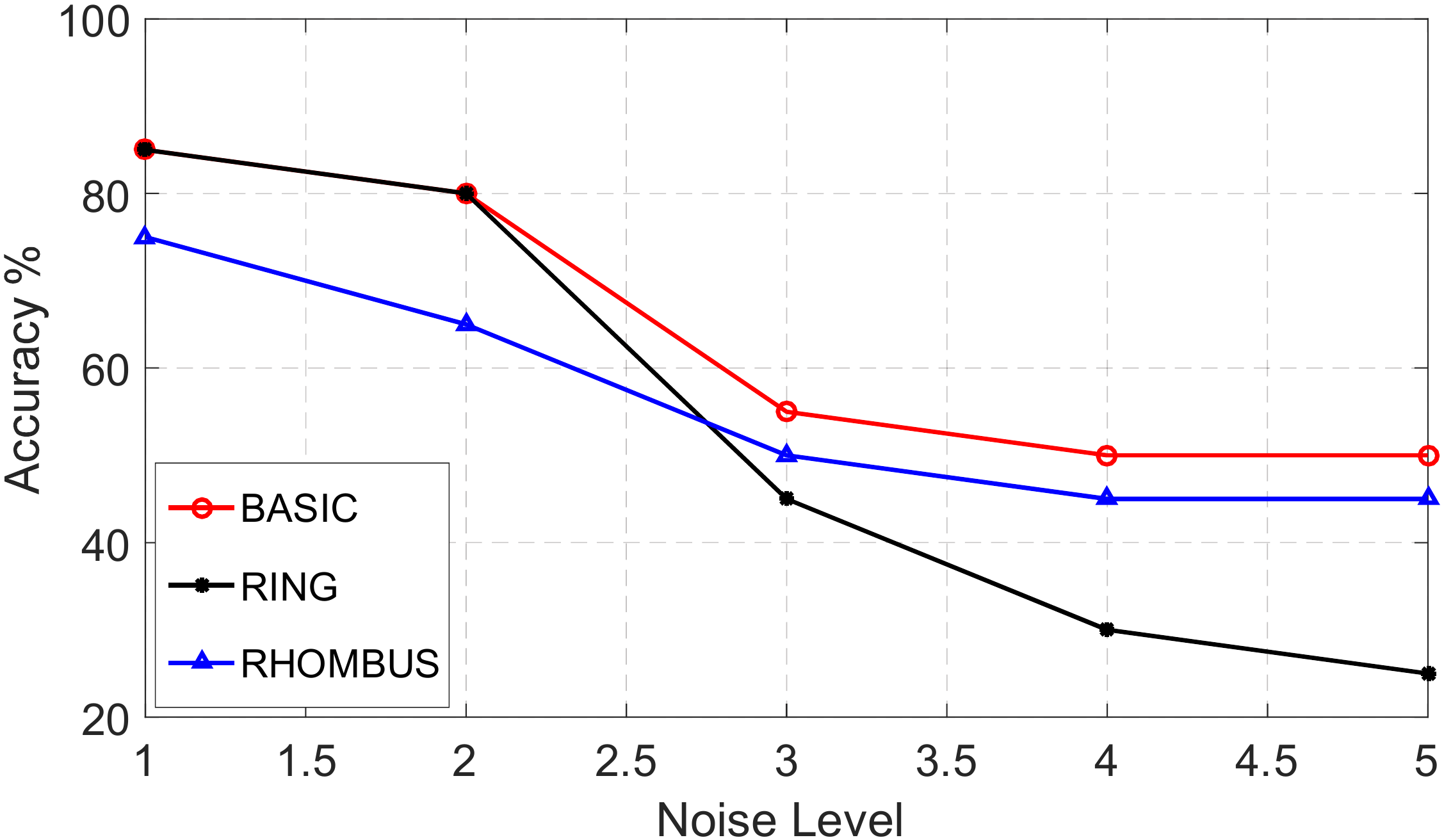}}~\vspace{0.1cm}
 \subfigure[Sensitivity to Size of $\mathcal{WS}$]{\includegraphics[scale =0.19]{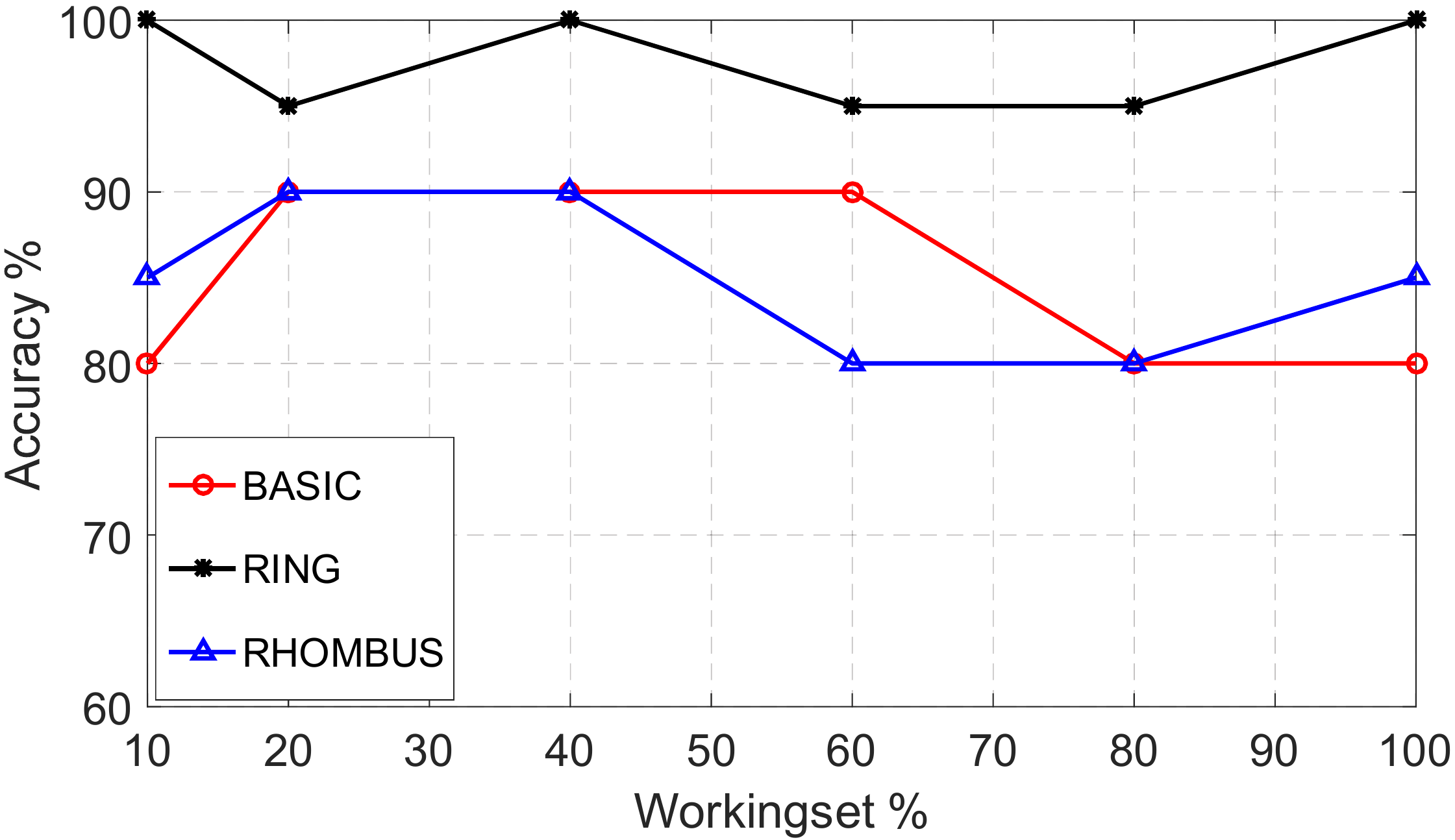}}~\vspace{-0.2cm}
 \caption{(a) TPI detection accuracy of PIGMIL with different noise levels on BASIC, RHOMBUS, and RING. Specifically, noise level '3' indicates that the labels of 20\% positive instances are changed into negative ones and the labels of 30\% negative instances are changed into positive ones. (b) TPI detection accuracy of PIGMIL with different sizes of $working \ set$ ($\mathcal{WS}$) on BASIC, RHOMBUS, and RING. '40\%' for $\mathcal{WS}$ indicates the size of $\mathcal{WS}$ is set to be 40\% of its corresponding $working \ bag$ ($\mathcal{WB}$).}
 \label{fig:dis_noise_ws}\vspace{-0.3cm}
\end{figure}

To evaluate {PIGMIL}'s ability to cope with noise, we added some noise to BASIC, RHOMBUS, and RING. In Figure \ref{fig:dis_noise_ws} (a), the noise level indicates how many {instances}' labels are changed.

According to Figure \ref{fig:dis_noise_ws} (a), the accuracy decreased when noise level increased. However, the decrease of accuracy was slowed down when noise level became bigger (e.g., the decrease of accuracy when noise level changed into '5' from '4' was smaller than that when changed into '2' from '3'), which demonstrated {PIGMIL}'s ability to cope with noise.
Moreover, accuracy decreased more sharply on RING than that on BASIC and RHOMBUS. This was because TPIs of RING are more centralized and show a greater difference from negative instances than TPIs of BASIC and RHOMBUS. So it was more hard for PIGMIL to detect TPIs on RING if TPIs were labelled negative.

\subsection{Sensitivity to Size of Working Set ($\mathcal{WS}$)}\label{dis:sub3}

We changed the size of $\mathcal{WS}$ to evaluate the influences of $working \ set$ ($\mathcal{WS}$) with different size on {PIGMIL}'s detection accuracy of TPIs.

Figure \ref{fig:dis_noise_ws} (b) reports the accuracy for different size of $\mathcal{WS}$ on BASIC, RHOMBUS, and RING and the $working \ set$ (\%) indicates the size of $\mathcal{WS}$ compared to its corresponding $working \ bag$ ($\mathcal{WB}$).
For BASIC and RHOMBUS, the highest accuracy was reached when the size of $\mathcal{WS}$ was about 40\% (of a positive bag). This was because some instances in positive bags are the false positive instances (FPIs) that can provide little information to detect TPIs if they are included into $\mathcal{WS}$.
For RING, the accuracy did not change significantly when the size of $\mathcal{WS}$ changed. This was because TPIs in RING are obviously different from negative instances (including FPIs). So TPIs will be included into $\mathcal{WS}$  successfully even if the size of $\mathcal{WS}$ is small, let alone if the size of $\mathcal{WS}$ is big.

\section{Conclusion}\label{conl}

Positive instance detection is key to MIL. Various methods have been developed for this issue while suffering some disadvantages, such as ignoring global similarity among positive instances and irrelevance between negative ones. To this end, a positive instance detection via graph updating for multiple instance learning (PIGMIL) is proposed.
PIGMIL first constructs positive candidate pool (PCP) from $working \ sets$ ($\mathcal{WS}s$) of some $working \ bags$ ($\mathcal{WB}s$) to transform positive instance detection into an optimization problem. Then based on a consistent similarity and discrimination graph (CSDG), this problem is solved approximately by an instance updating strategy. Finally a bag classification scheme is constructed to classify a new bag. Extensive experiments demonstrated {PIGMIL}'s great ability to detect $TPIs$ and that it outperformed other baseline methods.

%\section*{Acknowledgment}

%\begin{thebibliography}{1}
%\end{thebibliography}

\bibliographystyle{IEEEtran}
\bibliography{IEEEabrv,xu16}

\end{document}